%% file: main.tex
\documentclass{article}

\usepackage{fullpage}
\usepackage{natbib}

\usepackage[utf8]{inputenc} 
\usepackage[T1]{fontenc}    
\usepackage{hyperref}       
\usepackage{url}            
\usepackage{booktabs}       
\usepackage{amsfonts}       
\usepackage{nicefrac}       
\usepackage{microtype}      
\usepackage{xcolor}         

\usepackage{amsmath,amssymb,amsthm}  
\usepackage{sty/wdel}
\usepackage{sty/rl-notations}
\usepackage[ruled]{algorithm2e}

\newcommand{\fakefootnote}{\protect\phantom{\footnotesize 1}\textsuperscript{,}}

\usepackage{wrapfig}

\usepackage{minitoc}

\newtheorem{theorem}{Theorem}

\newtheorem{lemma}{Lemma}

\newtheorem{remark}{Remark}

\usepackage{hyperref}  
\hypersetup{
  colorlinks,
  citecolor=blue,
  linkcolor=blue,
  urlcolor=blue
}

\newcommand{\mdvi}{\hyperref[algo:mdvi]{\texttt{MDVI}}\xspace}
\newcommand{\qlearning}{\hyperref[algo:qlearning]{\texttt{Q-LEARNING}}\xspace}
\newcommand{\Eone}{\hyperref[lemma:E_k bound]{\cE_1}\xspace}
\newcommand{\Etwo}{\hyperref[lemma:eps_k bound]{\cE_2}\xspace}
\newcommand{\Ethree}{\hyperref[lemma:refined E_k bound]{\cE_3}\xspace}
\newcommand{\Efour}{\hyperref[lemma:refined eps_k bound]{\cE_4}\xspace}
\newcommand\numeq[2]%
  {\stackrel{\scriptscriptstyle(\mkern-1.5mu#2\mkern-1.5mu)}{#1}}

\usepackage[capitalize,noabbrev,nameinlink]{cleveref}

\crefname{algorithm}{Algorithm}{Algorithms}
\crefname{assumption}{Assumption}{Assumptions}
\crefname{corollary}{Corollary}{Corollaries}
\crefname{definition}{Definition}{Definitions}
\crefname{equation}{Equation}{Equations}
\crefname{example}{Example}{Examples}
\crefname{figure}{Figure}{Figures}
\crefname{lemma}{Lemma}{Lemmas}
\crefname{proposition}{Proposition}{Propositions}
\crefname{remark}{Remark}{Remarks}
\crefname{table}{Table}{Tables}
\crefname{theorem}{Theorem}{Theorems}

\Crefname{algorithm}{Algorithm}{Algorithms}
\Crefname{assumption}{Assumption}{Assumptions}
\Crefname{corollary}{Corollary}{Corollaries}
\Crefname{definition}{Definition}{Definitions}
\Crefname{equation}{Equation}{Equations}
\Crefname{example}{Example}{Examples}
\Crefname{figure}{Figure}{Figures}
\Crefname{lemma}{Lemma}{Lemmas}
\Crefname{proposition}{Proposition}{Propositions}
\Crefname{remark}{Remark}{Remarks}
\Crefname{table}{Table}{Tables}
\Crefname{theorem}{Theorem}{Theorems}

\usepackage{autonum}
\makeatletter
\autonum@generatePatchedReferenceCSL{Cref}
\makeatother

\title{KL-Entropy-Regularized RL with a \\ Generative Model is Minimax Optimal}
\date{}

\author{
  Tadashi Kozuno\thanks{
  Correspondence: \texttt{tadashi.kozuno@gmail.com}.
  \textsuperscript{1} University of Alberta,
  \textsuperscript{2} Peking University,
  \textsuperscript{3} Google Research, Brain team,
  \textsuperscript{4} Universit\'e de Lorraine, CNRS, INRIA, IECL,
  \textsuperscript{5} University of Tokyo,
  \textsuperscript{6} DeepMind,
  \textsuperscript{7} Otto von Guericke University Magdeburg.
  }\fakefootnote\textsuperscript{1},
  Wenhao Yang\textsuperscript{2},
  Nino Vieillard\textsuperscript{3,4},
  Toshinori Kitamura\textsuperscript{5},
  \\
  Yunhao Tang\textsuperscript{6},
  Jincheng Mei\textsuperscript{3},
  Pierre M\'enard\textsuperscript{7},
  Mohammad Gheshlaghi Azar\textsuperscript{6},
  \\
  Michal Valko\textsuperscript{6},
  R\'emi Munos\textsuperscript{6},
  Olivier Pietquin\textsuperscript{3},
  Matthieu Geist\textsuperscript{3},
  Csaba Szepesv\'ari\textsuperscript{1,6}
}

\begin{document}

\maketitle

\begin{abstract}
    \looseness=-1
    In this work, we consider and analyze the sample complexity of model-free reinforcement learning with a generative model.
    Particularly, we analyze mirror descent value iteration (MDVI) by
    \citet{geist2019theory} and \citet{vieillard2020leverage},
    which uses the Kullback-Leibler divergence and entropy regularization
    in its value and policy updates.
    Our analysis shows that it is nearly minimax-optimal for finding
    an $\varepsilon$-optimal policy when $\varepsilon$ is sufficiently small.
    This is the first theoretical result that demonstrates that
    a simple model-free algorithm without variance-reduction
    can be nearly minimax-optimal under the considered setting.
\end{abstract}

\doparttoc 
\faketableofcontents 

\input{section/introduction}
\input{section/literature}
\input{section/preliminary}
\input{section/main_result}
\input{section/main_result_proof}
\input{section/illustrations}

\input{section/conclusion}

\bibliography{refs}
\bibliographystyle{sty/iclr2021_conference}

\newpage

\appendix

\part{Appendix}
\parttoc
\newpage
\input{appendix/notations}
\input{appendix/equivalence_proof}
\input{appendix/auxiliary_lemmas}
\input{appendix/tools_from_prob_theory}
\input{appendix/total_variance}
\input{appendix/lemma_proof}

\input{appendix/illustration_details}

\end{document}

%% file: section/introduction.tex
\section{Introduction}

In the generative model setting, the agent has access to a simulator
of a Markov decision process (MDP), to which the agent can query next
states of arbitrary state-action pairs \citep{azar2013minimax}.
The agent seeks a near-optimal policy using as small number of queries
as possible.

While the generative model setting is simpler than the
online reinforcement learning (RL) setting, proof techniques
developed under this setting often generalize to more complex settings.
For example, the total-variance technique developed by \citet{azar2013minimax}
and \citet{lattimore2012pac} is now an indispensable tool for a sharp analysis
of RL algorithms in the online RL setting
for tabular MDP \citep{azar2017minimaxRegret,jin2018isQlearning}
and linear function approximation \citep{zhou2021nearlyMinimax}.

In this paper, we consider a model-free approach for the generative model setting with tabular MDP.
Particularly, we analyze mirror descent value iteration (MDVI) by
\citet{geist2019theory} and \citet{vieillard2020leverage},
which uses Kullback-Leibler (KL) divergence and entropy regularization
in its value and policy updates.
We prove its near minimax-optimal sample complexity for finding
an $\varepsilon$-optimal policy when $\varepsilon$ is sufficiently small.
Our result and analysis have the following consequences.

First, we demonstrate the effectiveness of KL and entropy regularization.
There are some previous works that argue the benefit of regularization
from a theoretical perspective in value-iteration-like algorithms \citep{kozuno2019theoretical,vieillard2020leverage,vieillard2020munchausen} and
policy optimizaiton \citep{mei2020globalConv,cen2021fast,lan2022policy}.
Compared to those works, we show that simply combining value iteration with
regularization achieves the near minimax-optimal sample complexity.

Second, as discussed by \citet{vieillard2020leverage}, MDVI encompasses
various algorithms as special cases or equivalent forms.
While we do not analyze each algorithm, most of them
are minimax-optimal too in the generative model setting with tabular MDP.

Lastly and most importantly, MDVI uses no variance-reduction technique, in contrast to previous model-free approaches \citep{sidford2018nearOptimal,wainwright2019variance,khamaru2021instance}.
Consequently, our analysis is straightforward, and it would be easy
to extend it to more complex settings, such as the online RL and linear function approximation.
Furthermore, previous approaches need pessimism to obtain a near-optimal policy,
which prevents them from being extended to the online RL setting,
where the optimism plays an important role for an efficient exploration
\citep{azar2017minimaxRegret,jin2018isQlearning}.
On the other hand, MDVI is compatible with optimism.
Our analysis paves the way for the combination of online exploration techniques with minimax model-free algorithms.

%% file: section/literature.tex
\section{Related work}\label{sec:related work}

Write $\gamma$, $H$, $X$, and $A$ for the discount factor, effective horizon $\frac{1}{1-\gamma}$, and number of states and actions.

\paragraph{Learning with a generative model}

In the generative model setting, there are two problem settings:
finding (i) an $\varepsilon$-optimal Q-value function with probability at least
$1-\delta$,
and (ii) an $\varepsilon$-optimal policy with probability at least
$1-\delta$, where $\delta \in (0, 1)$, and $\varepsilon > 0$.
Both problems are known to have sample complexity lower bounds of $\Omega(XAH^3 / \varepsilon^2)$ \citep{azar2013minimax,sidford2018nearOptimal}.
Note that even if an $\varepsilon$-optimal Q-value function is obtained,
additional data and computation are necessary to find an $\varepsilon$-optimal policy \citep{sidford2018nearOptimal}.
In this paper, we consider the learning of an $\varepsilon$-optimal policy.

There exist minimax-optimal model-based algorithms for learning a near-optimal
value function \citep{azar2013minimax} and policy \citep{agarwal2020modeBased, li2020breaking}.
Also, there exist minimax-optimal model-free algorithms for learning a near-optimal
value function \citep{wainwright2019variance,khamaru2021instance,li2021polyak} and policy \citep{sidford2018nearOptimal}.
While model-based algorithms are conceptually simple,
they have a higher computational complexity than that of model-free algorithms.
The algorithm (\mdvi) we analyze in this paper is a model-free algorithm
for finding a near-optimal policy, and has a low computational complexity.

Arguably, Q-learning is one of the simplest model-free algorithms \citep{watkins1992qlearning,even2003learning}.
Unfortunately, \citet{li2021q} provide a tight analysis of
Q-learning and show that its sample complexity is $\widetilde{\cO}( XAH^4 / \varepsilon^2 )$
for finding an $\varepsilon$-optimal Q-value function,
\footnote{
$\widetilde{\cO}$ hides terms poly-logarithmic in $H$, $X$, $A$, $1/\varepsilon$, and $1/\delta$.}
which
is one $H$ factor away from the lower bound.
To remove the extra $H$ factor, some works \citep{sidford2018nearOptimal,wainwright2019variance,khamaru2021instance} leverage variance reduction techniques.
While elegant, variance reduction techniques lead to multi-epoch algorithms
with involved analyses.
In contrast, \mdvi requires no variance reduction and is significantly simpler.

\mdvi's underlying idea that enables such simplicity is, while implicit, the averaging of value function estimates.
\cite{li2021polyak} shows that averaging Q-functions computed in Q-learning
can find a near-optimal Q-function with a minimax-optimal sample complexity.
\citet{azar2011speedyQ} also provides a simple algorithm called Speedy Q-learning (SQL), which performs the averaging of value function estimates.
In fact, as argued in \citep{vieillard2020leverage},
SQL is equivalent to a special case of MDVI with only KL regularization.
While previous works on MDVI \citep{vieillard2020leverage}
and an equivalent algorithm called CVI \citep{kozuno2019theoretical}
provide error propagation analyses,
they do not provide sample complexity.
\footnote{\citet{vieillard2020leverage} note SQL's
sample complexity of $\widetilde{\cO} (XA H^4 / \varepsilon^2)$ for finding a near-optimal policy as a corollary of their result without proof.}
This paper proves the first nearly minimax-optimal sample complexity bound for MDVI-type algorithm.
We tighten previous results by
(i) using the entropy regularization, which speeds up the convergence rate,
(ii) improved error propagation analyses (\cref{lemma:non-stationary error propagation,lemma:error propagation}),
and (iii) careful application of the total variance technique \citep{azar2013minimax}.

In addition to the averaging, \cref{theorem:non-stationary pac bound} is based on
the idea of using a non-stationary policy \citep{scherrer2012use}.
While the last policy of \mdvi is near-optimal when $\varepsilon$ is small,
a non-stationary policy constructed from policies outputted by \mdvi is near-optimal
for a wider range of $\varepsilon$.

\paragraph{Range of Valid $\varepsilon$}

Although there are multiple minimax-optimal algorithms for the generative model setting,
there ranges of valid $\varepsilon$ differ.
The model-based algorithm by \citet{azar2013minimax} is nearly minimax-optimal for $\varepsilon \leq \sqrt{H / \aX}$, which is later improved to $\sqrt{H}$ by \citet{agarwal2020modeBased},
and to $H$ by \citet{li2020breaking}.
As for model-free approaches, the algorithm
by \citet{sidford2018nearOptimal} is nearly minimax-optimal for $\varepsilon \leq 1$.
\mdvi is nearly minimax-optimal for $\varepsilon \leq 1/\sqrt{H}$
(non-stationary policy case, \cref{theorem:non-stationary pac bound})
and $\varepsilon \leq 1/H$ (last policy case, \cref{theorem:pac bound}).
Therefore, it has one of the narrowest range of valid $\varepsilon$ (second worst) compared to other algorithms.
It is unclear if this is an artifact of our analysis or the real limitation of
MDVI-type algorithm.
We leave this topic as a future work.

\paragraph{Regularization in MDPs}

Sometimes, regularization is added to the reward to encourage exploration in MDPs \citep{fox2015taming, vamplew2017softmax}. In recent years, \cite{neu2017unified, geist2019theory, lee2018sparse, yang2019regularized} have provided a unified framework for regularized MDPs. Specifically, \cite{geist2019theory} propose the Regularized Modified Policy Iteration algorithm and Mirror Descent Modified Policy Iteration to solve regularized MDPs. In the meantime, \cite{vieillard2020leverage} provide theoretical guarantees of KL-regularized value iteration in the approximate setting.
Particularly, they show that KL regularization results in
the averaging of Q-value functions and show that
the averaging leads to an improved error propagation result.
We extend their improved error propagation result to a KL- and entropy- regularization case.
Our results provide theoretical underpinnings to
many regularized RL algorithms in \citet[Table~1]{vieillard2020leverage}
and a high-performing deep RL algorithm called Munchausen DQN \citep{vieillard2020munchausen}.

%% file: section/preliminary.tex
\section{Preliminaries}

For a set $\bS$, we denote its complement as $\bS^c$.
For a positive integer $N$, we let $\brack*{N} \df \{1, \ldots, N \}$.
Without loss of generality, every finite set is assumed to be a subset of integers.
For a finite set, say $\bS$, the set of probability distributions over $\bS$ is denoted by $\Delta(\bS)$.
For a vector $v \in \R^M$, its $m$-th element is denoted by $v_m$ or $v(m)$.
\footnote{Unless noted otherwise, all vectors are column vectors.}
We let $\bone \df (1, \ldots, 1)^\top$ and $\bzero \df (0, \ldots, 0)^\top$, whose dimension will be clear from the context.
For a matrix $A \in \R^{N \times M}$, we denote its $n$-th row and $m$-th value of the $n$-th row by $A_n$ and $A_n^m$, respectively.
The expectation and variance of a random variable $X$ are denoted as $\E [X]$ and $\Var [X]$, respectively.
The empty sum is defined to be $0$, e.g., $\sum_{i=j}^k a_i = 0$ if $j > k$.

We consider a Markov Decision Process (MDP) defined by $\paren*{\X, \A, \gamma, r, P}$,
where $\X$ is the state space of size $X$,
$\A$ the action space of size $A$,
$\gamma \in [0, 1)$ the discount factor,
$r \in [-1, 1]^{\aXA}$ the reward vector with $r_{x, a}$ denoting the reward
when taking an action $a$ at a state $x$,
and $P \in \R^{\aXA \times \aX}$ state transition probability matrix with $P_{x, a}^y$
denoting the state transition probability to a new state $y$ from a state
$x$ when taking an action $a$.
We let $H$ be the (effective) time horizon $1 / (1-\gamma)$.

Note that $(P v) (x, a) = E \brackc*{v (X_1)}{X_0=x, A_0=a}$ for any $v \in \R^{\aX}$. Any policy $\pi$ is identified as a matrix $\pi \in \R^{\aX \times \aXA}$ such that $(\pi q)(x) := \sum_{a \in \A} \pi(a|x) q \paren*{x, a}$ for any $q \in \R^{\aXA}$. For convenience, we adopt a shorthand notation, $P^{\pi} := P \pi$. With these notations, the Bellman operator $T^{\pi}$ for a policy $\pi$ is defined as an operator such that $T^{\pi} q := r + \gamma P^{\pi} q$. The Q-value function $\qf{\pi}$ for a policy $\pi$ is its unique fixed point. The state-value function $\vf{\pi}$ is defined as $\pi \qf{\pi}$. An optimal policy $\pi_*$ is a policy such that $\vf{*} := \vf{\pi_*} \geq \vf{\pi}$ for any policy $\pi$, where the inequality is point-wise.

%% file: section/main_result.tex
\section{Mirror Descent Value Iteration and Main Results}\label{sec:mdvi}

For any policies $\pi$ and $\mu$, let $\log \pi$ and $\log \frac{\pi}{\mu}$ be the functions $x,a \mapsto \log \pi (a|x)$ and $x,a \mapsto \log \frac{\pi (a|x)}{\mu (a|x)}$ over $\X\times\A$.
We analyze (approximate) MDVI whose update is the following \citep{vieillard2020leverage}:
\begin{gather}\label{eq:MDVI update}
    q_{k+1}
    =
    r + \gamma P v_k + \varepsilon_k\,,
\end{gather}
where
$
    v_k \df \pi_k \paren*{
        q_k - \tau \log \dfrac{\pi_k}{\pi_{k-1}} - \kappa \log \pi_k
    }
$,
\begin{align}\label{eq:policy update}
    \pi_k \parenc*{\cdot}{x}
    =
    \argmax_{p \in \Delta(\A)}
    \sum_{a \in \A} p (a) \paren*{
        q_k (s, a) - \tau \log \frac{p (a)}{\pi_{k-1} \parenc*{a}{x}} - \kappa \log p (a)
    }
\end{align}
for all $x \in \X$,
and $\varepsilon_k: \XA \rightarrow \R$ is an ``error'' function,
which abstractly represents the deviation of $q_{k+1}$ from
the update target $r + \gamma P v_k$.
In other words, MDVI is value iteration with KL and entropy regularization.

Let $s_k \df q_k + \alpha s_{k-1} = \sum_{j=0}^{k-1} \alpha^j q_{k-j}$.
The policy \eqref{eq:policy update} can be rewritten as a Boltzmann policy of $s_k$, i.e.,
$
    \pi_k (a|x)
    \propto
    \exp \paren*{ \beta s_k (x, a) }
$,
where $\alpha \df \tau / (\tau + \kappa)$, and $\beta \df 1 / (\tau + \kappa)$,
 see \cref{sec:equivalence proof} for details. 
Substituting $\pi_{k-1}$ and $\pi_k$ in $v_k$ with this expression of the policy, we deduce that
\begin{align}
    v_k (x)
    =
    \frac{1}{\beta} \log \sum_{a \in \A} \exp \paren*{ \beta s_k (x, a) }
    -
    \frac{\alpha}{\beta} \log \sum_{a \in \A} \exp \paren*{ \beta s_{k-1} (x, a) }\,.
\end{align}
Thus, letting $w_k$ be the function $x \mapsto \beta^{-1} \log \sum_{a \in \A} \exp \paren*{ \beta s_k (x, a) }$ over $\X$,
MDVI's update rules can be equivalently written as
\begin{gather}
    q_{k+1}
    =
    r + \gamma P \paren*{
        w_k - \alpha w_{k-1}
    } + \varepsilon_k
    \text{ and }
    \pi_k \parenc{a}{x}
    \propto
    \exp \paren*{ \beta s_k (x, a) }
    \text{ for all } (x, a) \in \XA\,.
\end{gather}
A sample-approximate version of MDVI shown in \cref{algo:mdvi} (\mdvi)
uses this equivalent form of MDVI.
Furthermore, for simplicity of the analysis, we consider the limit of $\tau, \kappa \to 0$ while keeping $\alpha = \tau / (\tau+\kappa)$ to a constant value (which corresponds to letting $\beta\to\infty$).

\begin{algorithm}[t!]
    \KwIn{$\alpha \in [0, 1)$, number of iterations $K$, and number of next-state samples per iteration $M$.}
    Let $s_0 = \bzero \in \R^{\aXA}$ and $w_0 = w_{-1} = \bzero \in \R^{\aX}$\;
    \For{$k$ \textbf{\emph{from}} $0$ \textbf{\emph{to}} $K-1$}{
        Let $v_k = w_k - \alpha w_{k-1}$\;
        \For{\textbf{\emph{each state-action pair}} $\paren*{x, a} \in \XA$}{
            Sample $(y_{k, m, x, a})_{m=1}^{M}$ from the generative model $P(\cdot|x, a)$\;
            Let $q_{k+1} (x, a) = r (x, a) + \gamma M^{-1} \sum_{m=1}^M v_k (y_{k, m, x, a})$\;
        }
        Let $s_{k+1} = q_{k+1} + \alpha s_k$ and $w_{k+1} (x) = \max_{a \in \A} s_{k+1} (x, a)$ for each $x \in \X$\;
    }
    \Return{$(\pi_k)_{k=0}^K$ , where $\pi_k$ is greedy policy with respect to $s_k$\;}
    \caption{$\mdvi (\alpha, K, M)$}\label{algo:mdvi}
\end{algorithm}

\begin{remark}
    Even if $\beta$ is finite, \mdvi is nearly minimax-optimal
    as long as $\beta$ is large enough.
    Indeed, $\beta^{-1} \log \sum_{a \in \A} \exp (q (x, a))$ satisfies \citep[Lemma~7]{kozuno2019theoretical} that
    \[
        \max_{a \in \A} q (x, a)
        \leq
        \beta^{-1} \log \sum_{a \in \A} \exp (q (x, a))
        \leq
        \max_{a \in \A} q (x, a) + \beta^{-1} \log A\,.
    \]
    Thus, while $\beta$ appears in the proofs of \cref{theorem:non-stationary pac bound,theorem:pac bound} if it is finite, it always appear as $\beta^{-1} \log A$ multiplied by $H$-dependent constant. Therefore,
    \mdvi is nearly minimax-optimal as long as $\beta$ is large enough.
\end{remark}

\paragraph{Why KL Regularization?}
The weight $\alpha$ used in $s_k$ updates monotonically increases as the coefficient of the KL regularization $\tau$ increases.
As we see later, error terms appear in upper bounds of $\infnorm{\vf{*} - \vf{\pi_k}}$ as $(1 - \alpha) \sum_{j=1}^k \alpha^{k-j} \varepsilon_j$.
Applying Azuma-Hoeffiding inequality, it is approximately bounded by $H \sqrt{1-\alpha}$.
Therefore, \mdvi becomes more robust to sampling error as $\alpha$ increases.
The KL regularization confers this benefit to the algorithm.

\looseness=-1
\paragraph{Why Entropy Regularization?}
When there is no entropy regularization ($\alpha=1$),
the convergence rate of MDVI becomes $1/K$ while it
is $\alpha^K$ for $\gamma \leq \alpha < 1$ \citep{vieillard2020leverage}.
In the former case, we need to set $K \approx H^2 / \varepsilon$,
whereas in the latter case, $K \approx 1 / (1 - \alpha)$ suffices.
Since we will set $\alpha$ to either $\gamma$ or $1 - (1-\gamma)^2$,
$K \approx H$ or $H^2$.
Thus, we can use more samples per one value update (i.e., larger $M$).
A larger $M$ leads to a smaller value estimation variance ($\sigma (v_k)$ in \cref{lemma:variance upper bounds}),
which is important to improve the range of $\varepsilon$.
Even when $\alpha=1$, \mdvi is nearly minimax-optimal (proof omitted).
However, $\varepsilon$ must be less than or equal to $1/H^2$.

\paragraph{Main Theoretical Results}
The following theorems show the near minimax-optimality of \mdvi.
For a sequence of policies $(\pi_k)_{k=0}^K$ outputted by \mdvi,
we let $\pi_k'$ be the non-stationary policy that follows $\pi_{k-t}$ at the $t$-th time step until $t=k$,
after which $\pi_0$ is followed.
\footnote{The time step index $t$ starts from $0$.}
Note that the value function of such a non-stationary policy is given by $\vf{\pi'_k} = \pi_k T^{\pi_{k-1}} \cdots T^{\pi_1} \qf{\pi_0}$.

\begin{theorem}\label{theorem:non-stationary pac bound}
    Assume that $\varepsilon \in (0, 1 / \sqrt{H}]$.
    Then, there exist positive constants $c_1, c_2 \geq 1$ independent
    of $H$, $\aX$, $\aA$, $\varepsilon$, and $\delta$ such that
    when \mdvi is run with the settings
    \begin{align}
        \alpha = \gamma\,,
        K = \ceil*{ \frac{3}{1-\alpha} \log \frac{c_1 H}{\varepsilon} + 2 }\,,
        \text{ and }
        M = \ceil*{ \frac{c_2 H^2}{\varepsilon^2} \log \frac{16 K \aXA}{\delta} }\,,
    \end{align}
    it outputs a sequence of policies $(\pi_k)_{k=0}^K$ such that
    $
        \infnorm{\vf{*} - \vf{\pi'_K}}
        \leq
        \varepsilon
    $
    with probability at least $1 - 3 \delta / 4$, using $\widetilde{\cO} \paren*{H^3 \aXA / \varepsilon^2}$ samples from the generative model.
\end{theorem}

Storing all policies requires the memory space of $KXA$ and can be prohibitive in some cases.
The next theorem shows that the last policy outputted by \mdvi is near-optimal when $\varepsilon \leq 1/H$.

\begin{theorem}\label{theorem:pac bound}
    Assume that $\varepsilon \in (0, 1 / H]$.
    Then, there exist positive constants $c_3, c_4 \geq 1$ independent
    of $H$, $\aX$, $\aA$, $\varepsilon$, and $\delta$ such that
    when \mdvi is run with the settings
    \begin{align}
        \alpha = 1 - (1 - \gamma)^2\,,
        K = \ceil*{ \frac{5}{1-\alpha} \log \frac{c_3 H}{\varepsilon} + 2 }\,,
        \text{ and }
        M = \ceil*{ \frac{c_4 H}{\varepsilon^2} \log \frac{16 K \aXA}{\delta} }\,,
    \end{align}
    it outputs a sequence of policies $(\pi_k)_{k=0}^K$ such that
    $
        \infnorm{\vf{*} - \vf{\pi_K}}
        \leq
        \varepsilon
    $
    with probability at least $1 - \delta$, using $\widetilde{\cO} \paren*{H^3 \aXA / \varepsilon^2}$ samples from the generative model.
\end{theorem}

%% file: section/main_result_proof.tex
\section{Proofs of the Main Results}\label{sec:main result}

Before the proof, we introduce some notations.
A table of notations is provided in \cref{app:notations}.

\paragraph{Notation.}
$\square$ denotes an indefinite constant that changes throughout the proof
and is independent of $H$, $\aX$, $\aA$, $\varepsilon$, and $\delta$.
We let $A_{\gamma, k} \df \sum_{j=0}^{k-1} \gamma^{k-j} \alpha^j$ and $A_k \df \sum_{j=0}^{k-1} \alpha^j$ for any non-negative integer $k$
with $A_\infty \df 1 / (1-\alpha)$.
$\bF_{k, m}$ denotes the $\sigma$-algebra generated by random variables $\{ y_{j, n, x, a} | (j, n, x, a) \in [k-2] \times [M] \times \XA \} \cup \{ y_{j, n, x, a} | (j, n, x, a) \in \{k-1\} \times [m-1] \times \XA \}$.
For any $k \in \{0\} \cup [K-1]$ and $v \in \R^{\aX}$, $\PVar(v)$ and $\widehat{P}_k v$ denote the functions
\begin{align}
    \textstyle 
    \PVar(v): (x, a) \mapsto (P v^2) (x, a) - ( P v )^2 (x, a)
    \text{ and }
    \widehat{P}_k v: (x, a) \mapsto \sum_{m=1}^{M} v (y_{k, m, x, a}) / M\,,
\end{align}
respectively. We often write $\sqrt{\PVar(v)}$ as $\sigma(v)$. Furthermore, $\varepsilon_k$ and $E_k$ denote ``error'' functions
\begin{align}
    \textstyle 
    \varepsilon_k: (x, a) \mapsto \gamma \widehat{P}_{k-1} v_{k-1} (x, a) - \gamma P v_{k-1} (x, a)
    \text{ and }
    E_k : (x, a) \mapsto \sum_{j=1}^k \alpha^{k-j} \varepsilon_j (x, a)\,,
\end{align}
respectively. (Note that $\varepsilon_1 = E_1 = \bzero$ since $v_0 = \bzero$.)
For a sequence of policies $(\pi_k)_{k \in \Z}$, we let $T_j^i \df T^{\pi_i} T^{\pi_{i-1}} \cdots T^{\pi_{j+1}} T^{\pi_j}$ for $i \geq j$, and $T_j^i \df I$ otherwise.
We also let $P_j^i \df P^{\pi_i} P^{\pi_{i-1}} \cdots P^{\pi_{j+1}} P^{\pi_j}$ for $i \geq j$, and $P_j^i \df I$ otherwise.
As a special case with $\pi_k = \pi_*$ for all $k$, we let $P_*^i \df (P^{\pi_*})^i$.
Finally, throughout the proof, $\iota_1$ and $\iota_2$ denotes
$
    \log (8 K \aXA / \delta)
$
and
$
    \log (16 K \aXA / \delta)
$,
repspectively.

\subsection{
    Proof of \texorpdfstring{
        \cref{theorem:non-stationary pac bound}
    }{
        Theorem~\ref{theorem:non-stationary pac bound}
    }
    (Near-optimality of the Non-stationary Policy)
}
\input{proof/non_stationary_pac_proof}

\subsection{
    Proof of \texorpdfstring{
        \cref{theorem:pac bound}
    }{
        Theorem~\ref{theorem:pac bound}
    } (Near-optimality of the Last Policy)
}
\input{proof/pac_proof}

%% file: proof/non_stationary_pac_proof.tex
The first step of the proof is the error propagation analysis of MDVI given below. It differs from the one of \citet{vieillard2020leverage} since ours upper-bounds $\vf{*} - \vf{\pi'_k}$. It is proven in \cref{subsec:proof of non-stationary error propagation}

\begin{lemma}\label{lemma:non-stationary error propagation}
    For any $k \in [K]$,
    $
        \bzero
        \leq
        \vf{*} - \vf{\pi'_k}
        \leq
        \Gamma_k
    $,
    where
    \begin{align}
        \Gamma_k \df \displaystyle \frac{1}{A_\infty} \sum_{j=0}^{k-1} \gamma^j \paren*{
            \pi_k P_{k-j}^{k-1} - \pi_* P_*^j
        } E_{k-j}
        + 2 H \paren*{ \alpha^k + \frac{A_{\gamma, k}}{A_\infty} } \bone\,.
    \end{align}
\end{lemma}

From this result, it can be seen that an upper bound for each $E_k$ is necessary.
The following lemma provides an upper bound, which readily lead to \cref{lemma:non-stationary coarse bound} when combined with \cref{lemma:non-stationary error propagation}.
These lemmas are proven in \cref{subsec:proof of non-stationary coarse bound}.

\begin{lemma}\label{lemma:E_k bound}
    Let $\cE_1$ be the event that
    $
        \norm{E_k}_\infty
        <
        3 H \sqrt{A_\infty \iota_1 / M}
    $
    for all $k \in [K]$.
    Then, $\P \paren*{\cE_1^c} \leq \delta / 4$.
\end{lemma}

\begin{lemma}\label{lemma:non-stationary coarse bound}
    Assume that $\varepsilon \in (0, 1]$.
    When \mdvi is run with the settings $\alpha$, $K$, and $M$ in \cref{theorem:non-stationary pac bound},
    under the event $\Eone$,
    its output policies $(\pi_k)_{k=0}^K$ satisfy that
    $
        \infnorm{\vf{*} - \vf{\pi'_k}}
        \leq
        2 (k + H) \gamma^k + \square \varepsilon \sqrt{H / c_2}
    $
    for all $k \in [K]$.
    Furthermore, $\infnorm{\vf{*} - \vf{\pi'_K}} \leq \sqrt{H} \varepsilon$
    for some $c_1, c_2 \geq 1$.
\end{lemma}

Unfortunately, \cref{lemma:non-stationary coarse bound} is insufficient to show the minimax optimality of \mdvi since it only holds that $\infnorm{\vf{*} - \vf{\pi'_K}} \leq \sqrt{H} \varepsilon$ while $\cP(\cE_1) \geq 1 - \delta$.
Any other setting of $\alpha$, $\beta$, $K$, and $M$ does not seem to lead to $\infnorm{\vf{*} - \vf{\pi'_K}} \leq \varepsilon$.
Nonetheless, \cref{lemma:non-stationary coarse bound} turns out to be useful later to obtain a refined result.

To show the minimax optimality, we need to remove the extra $\sqrt{H}$ factor. The standard tools for this purpose are a Bernstein-type inequality and the total variance (TV) technique \citep{azar2013minimax},
which leverages the fact that
$
    \norm{\paren{I - \gamma P^\pi}^{-1} \sigma(\vf{\pi})}_\infty \leq \sqrt{2H^3}
$
for any policy $\pi$.
In our case, the TV technique for a non-stationary policy is required due to $\pi_k P_{k-j}^{k-1}$, though.

Recall the definition of $\varepsilon_k$ and note that its standard deviation consists of $\sigma (v_{k-1})$.
As we use a Bernstein inequality for martingale because of $E_k$, we derive an upper bound for the sum of $\sigma (v_{j-1})^2$ over $j \in [k]$
($V$ in \cref{lemma:conditional bernstein})
using the fact that
$\sigma (v_{j-1}) \approx \sigma (\vf{*})$ when $v_{j-1} \approx \vf{*}$.
To this end, the following lemma, proven in \cref{subsec:proof of v error prop}, is useful.

\begin{lemma}\label{lemma:v error prop}
    For any $k \in [K]$,
    \begin{align}
        - 2 \gamma^k H \bone - \sum_{j=0}^{k-1} \gamma^j \pi_{k-1} P_{k-j}^{k-1} \varepsilon_{k-j}
        \leq
        \vf{*} - v_k
        \leq
        \Gamma_{k-1} + 2 H \gamma^k \bone - \sum_{j=0}^{k-1} \gamma^j \pi_{k-1} P_{k-1-j}^{k-2} \varepsilon_{k-j}\,.
    \end{align}
\end{lemma}

Combining this lemma with \cref{lemma:E_k bound} and the following one, we can obtain an upper-bound for $\sigma (v_{k-1})$.
The proofs of both results are given in \cref{subsec:proof of coarse v bound}.

\begin{lemma}\label{lemma:eps_k bound}
    Let $\cE_2$ be the event that
    $
        \norm{\varepsilon_k}_\infty
        <
        3 H \sqrt{\iota_1/M}
    $
    for all $k \in [K]$.
    Then, $\P \paren*{\cE_2^c} \leq \delta / 4$.
\end{lemma}

\begin{lemma}\label{lemma:variance upper bounds}
    Conditioned on the event $\Eone \cap \Etwo$, it holds for any $k \in [K]$ that \begin{align}
        \sigma (v_k)
        \leq
        2 H \min \brace*{
            1,
            2 \max\{\alpha, \gamma\}^{k-1}
            + \frac{A_{\gamma, k-1}}{A_\infty}
            + 6 H \sqrt{\frac{\iota_1}{M}}
        } \bone
        + \sigma (\vf{*})\,.\label{eq:variance upper bounds}
    \end{align}
    Furthermore, $\sigma (v_0) = \bzero$.
\end{lemma}

Using \cref{lemma:variance upper bounds}, we can prove refined bounds for $E_k$ and $\varepsilon_k$, as in \cref{subsec:proof of refined bound}.

\begin{lemma}\label{lemma:refined E_k bound}
    Let $\cE_3$ be the event that
    \begin{align}
        \abs{E_k} (x, a)
        <
        \frac{4 H \iota_2}{3 M} + \sqrt{2 V_k (x, a) \iota_2}
        \text{ for all }
        (x, a, k) \in \XA \times [K]\,,
    \end{align}
    where
    $
        V_k
        \df
        4 \sum_{j=1}^k \alpha^{2 (k-j)} \overline{\PVar}_j / M
    $
    with
    \begin{align}
        \overline{\PVar}_j
        \df \PVar (\vf{*})
        + 4 H^2
        \paren*{
            4 \max \brace{\alpha, \gamma}^{2{j-2}}
            + \frac{A_{\gamma, j-2}^2}{A_\infty^2}
            + \frac{36 H^2 \iota_1}{M}
        } \bone
    \end{align}
    for $k \geq 2$ and $\overline{\PVar}_1 \df \bzero$.
    Then, $\P \parenc*{\cE_3^c}{\cE_1 \cap \cE_2} \leq \delta / 4$.
\end{lemma}

\begin{lemma}\label{lemma:refined eps_k bound}
    Let $\cE_4$ be the event that
    \begin{align}
        \abs{\varepsilon_k} (x, a)
        <
        \frac{4 H \iota_2}{3 M} + \sqrt{2 W_k (x, a) \iota_2}
        \text{ for all }
        (x, a, k) \in \XA \times [K]
    \end{align}
    where
    $
        W_k
        \df
        4 \overline{\PVar}_k / M
    $.
    Then, $\P \parenc*{\cE_4^c}{\cE_1 \cap \cE_2} \leq \delta / 4$.
\end{lemma}

With these lemmas, we are ready to prove \cref{theorem:non-stationary pac bound}.

\begin{proof}[Proof of \cref{theorem:non-stationary pac bound}]
    We condition the proof by $\Eone \cap \Etwo \cap \Ethree$.
    As for any events $A$ and $B$,
    $
        \P (A \cap B)
        = \P ( (A \cup B^c) \cap B)
        \geq 1 - \P (A^c \cap B) - \P (B^c)\,,
    $
    and $\P (A^c \cap B) = \P (A^c | B) \P (B) \leq \P (A^c | B)$,
    \begin{align}
        \P (\cE_1 \cap \cE_2 \cap \cE_3)
        &\geq 1 - \P (\cE_3^c | \cE_1 \cap \cE_2) - \P ((\cE_1 \cap \cE_2)^c)
        \\
        &\geq 1 - \P (\cE_3^c | \cE_1 \cap \cE_2) - \P (\cE_1^c) - \P (\cE_2^c)\,.
    \end{align}
    Therefore, from \cref{lemma:E_k bound,lemma:eps_k bound,lemma:refined E_k bound}, we conclude that
    $
        \P (\cE_1 \cap \cE_2 \cap \cE_3)
        \geq 1 - 3 \delta / 4\,.
    $
    Accordingly, any claim proven under $\cE_1 \cap \cE_2 \cap \cE_3$ holds with probability at least $1 - 3 \delta / 4$.
    
    From \cref{lemma:non-stationary error propagation},
    the setting that $\alpha = \gamma$,
    and the monotonicity of stochastic matrices,
    \begin{align}
        \vf{*} - \vf{\pi'_K}
        \leq
        \frac{1}{H} \underbrace{
            \sum_{k=0}^{K-1} \gamma^k \pi_* P_*^k \abs{E_{K-k}}
        }_{\heartsuit}
        + \frac{1}{H} \underbrace{
            \sum_{k=0}^{K-1} \gamma^k \pi_K P_{K-k}^{K-1} \abs{E_{K-k}}
        }_{\clubsuit}
        + 2 \paren*{ H + K } \gamma^K \bone\,.
    \end{align}
    As the last term is less than $\square \varepsilon / c_1$
    from \cref{lemma:k gamma to k-th inequality},
    it remains to upper-bound $\heartsuit$ and $\clubsuit$.
    We note that $A_{\infty} = H$ and $A_{\gamma, k} = k \gamma^k$ under the considered setting of $\alpha$.
    
    From the settings of $\alpha$ and $M$,
    \begin{align}
        2 V_k \iota_2
        &\leq
        \frac{\square \PVar (\vf{*}) \varepsilon^2}{c_2 H}
        + \frac{\square \varepsilon^2}{c_2} \paren[\Big]{
            k \gamma^{2(k-2)}
            + \frac{\gamma^{2(k-2)}}{H^2} \underbrace{
                \sum_{j=2}^k (j-2)^2
            }_{\leq \square k^3 \text{ from (a)}}
            + \frac{\varepsilon^2}{c_2} \underbrace{
                \sum_{j=2}^k \gamma^{2(k-j)}
            }_{\leq H}
        } \bone\,,
    \end{align}
    where (a) follows from \cref{lemma:sum of k power}.
    From this result and \cref{lemma:sqrt inequality}, it follows that
    \begin{align}
        \heartsuit
        &\leq
        \frac{\square \varepsilon^2}{c_2} \bone
        + \frac{\square \varepsilon}{\sqrt{c_2 H}} \underbrace{
            \sum_{k=0}^{K-1} \gamma^k \pi_* P_*^k \sigma (\vf{*})
        }_{\leq \sqrt{2 H^3} \bone \text{from \cref{lemma:total variance}}}
        + \frac{\square \varepsilon}{\sqrt{c_2}} \paren[\Big]{
            \gamma^{K-2} \underbrace{
                \sum_{k=1}^K \paren*{
                    \sqrt{k} + \frac{k \sqrt{k}}{H}
                }
            }_{\leq \square \paren*{K^{2.5} / H} \text{ from (a)}}
            + H \sqrt{H} \varepsilon
        } \bone\,,
    \end{align}
    where (a) follows from \cref{lemma:sum of k power}
    and that $H \leq K$.
    From \cref{lemma:k gamma to k-th inequality},
    $K^{2.5} \gamma^{K-2} / H \leq \square \varepsilon / c_1$.
    Therefore, using the inequality $\varepsilon \leq 1 / \sqrt{H} \leq 1$,
    $
        H^{-1} \heartsuit
        \leq
        \square \paren*{c_2^{-1} + c_2^{-0.5}} \varepsilon \bone
    $.
    
    Although an upper bound for $\clubsuit$ can be similarly derived,
    a care must be taken when upper-bounding $\diamondsuit \df \sum_{k=0}^{K-1} \gamma^k \pi_K P_{K-k}^{K-1} \sigma (\vf{*})$.
    From \cref{lemma:variance decomposition}, for any $k \in [K]$,
    \begin{align}
        \sigma (\vf{*})
        \leq
        \sigma (\vf{*} - \vf{\pi_k'}) + \sigma (\vf{\pi_k'})
        \leq
        2 (k+H) \gamma^k \bone
        + \square \sqrt{H / c_2} \varepsilon \bone
        + \sigma (\vf{\pi_k'})\,,
    \end{align}
    where the second inequality follows
    from \cref{lemma:popoviciu,lemma:non-stationary coarse bound}.
    Accordingly,
    \begin{align}
        \diamondsuit
        \leq
        2 \gamma^K \underbrace{
            \sum_{k=0}^{K-1} (k+H)
        }_{\leq \square K^2 \text{ from (a)}} \bone
        + \square H \sqrt{H / c_2} \varepsilon \bone
        + \underbrace{
            \sum_{k=0}^{K-1} \gamma^k \pi_K P_{K-k}^{K-1} \sigma (\vf{\pi_{K-k}'})
        }_{\leq \sqrt{2 H^3} \bone \text{ from \cref{lemma:total variance}}}
        \leq
        \square H \sqrt{H}\,,
    \end{align}
    where (a) follows from \cref{lemma:sum of k power}
    and that $H \leq K$, and the second inequality follows
    since $\varepsilon \leq 1 / \sqrt{H} \leq 1$ and
    $K^2 \gamma^K \leq \square \varepsilon / c_1$ from \cref{lemma:k gamma to k-th inequality}.
    Thus,
    $
        H^{-1} \clubsuit
        \leq
        \square \paren{c_2^{-1} + c_2^{-0.5}} \varepsilon \bone
    $.

    Combining these results, we conclude that there are constants
    $c_1$ and $c_2$ that satisfy the claim.
\end{proof}

%% file: proof/pac_proof.tex
We need the following error propagation result. Its proof is given in \cref{subsec:proof of error propagation}.

\begin{lemma}[Error Propagation of \mdvi]\label{lemma:error propagation}
    For any $k \in [K]$,
    \begin{align}
        \bzero
        \leq
        \vf{*} - \vf{\pi_k}
        &\leq
        2 H \paren*{\alpha^k + \frac{A_{\gamma, k}}{A_\infty}} \bone
        + \frac{1}{A_\infty} \paren*{\cN^{\pi_k} \pi_k - \cN^{\pi_*} \pi_*} E_k
        \\
        &\hspace{3em}+ \frac{1}{A_\infty} \sum_{j=1}^k \gamma^j\paren*{
            \cN^{\pi_*} \pi_* P_{k+1-j}^k - \cN^{\pi_k} \pi_k P_{k-j}^{k-1}
        } E_{k+1-j}'\,,\label{eq:error prop}
    \end{align}
    where $\cN^\pi \df \sum_{t=0}^\infty (\gamma \pi P)^t$
    for any policy $\pi$, and
    $E_{k+1-j}' \df \varepsilon_{k+1-j} - (1-\alpha) E_{k-j}$.
\end{lemma}

The following lemma is an analogue of \cref{lemma:non-stationary coarse bound}.
It is proven in \cref{subsec:proof of last policy coarse bound}.

\begin{lemma}\label{lemma:coarse bound for last policy}
    Assume that $\varepsilon \in (0, 1]$.
    When \mdvi is run with the settings $\alpha$, $K$, and $M$ in \cref{theorem:pac bound},
    under the event $\Eone \cap \Etwo$,
    its output policies $(\pi_k)_{k=0}^K$ satisfy that
    $
        \infnorm{\vf{*} - \vf{\pi'_k}}
        \leq
        \square H \alpha^k + \square \varepsilon \sqrt{H / c_4}
    $
    and
    $
        \infnorm{\vf{*} - \vf{\pi_k}}
        \leq
        \square H \alpha^k + \square \varepsilon \sqrt{H / c_4}
    $
    for all $k \in [K]$.
\end{lemma}

Now, we are ready to prove \cref{theorem:pac bound}.

\begin{proof}[Proof of \cref{theorem:pac bound}]
    We condition the proof by $\Eone \cap \Etwo \cap \Ethree \cap \Efour$.
    Since for any events $A$ and $B$,
    $
        \P (A \cap B)
        = \P ( (A \cup B^c) \cap B)
        \geq 1 - \P (A^c \cap B) - \P (B^c)
    $,
    and $\P (A^c \cap B) = \P (A^c | B) \P (B) \leq \P (A^c | B)$,
    \begin{align}
        \P (\cE_1 \cap \cE_2 \cap \cE_3 \cap \cE_4)
        &\geq 1 - \P ( (\cE_3 \cap \cE_4)^c | \cE_1 \cap \cE_2) - \P ((\cE_1 \cap \cE_2)^c)
        \\
        &\geq 1 - \P (\cE_3^c \cup \cE_4^c | \cE_1 \cap \cE_2) - \P (\cE_1^c) - \P (\cE_2^c)
        \\
        &\geq 1 - \P (\cE_3^c | \cE_1 \cap \cE_2) - \P (\cE_4^c | \cE_1 \cap \cE_2) - \P (\cE_1^c) - \P (\cE_2^c)\,.
    \end{align}
    Therefore, from \cref{lemma:E_k bound,lemma:eps_k bound,lemma:refined E_k bound,lemma:refined eps_k bound}, we conclude that
    $
        \P (\cE_1 \cap \cE_2 \cap \cE_3 \cap \cE_4)
        \geq 1 - \delta\,.
    $
    Accordingly, any claim proven under $\cE_1 \cap \cE_2 \cap \cE_3 \cap \cE_4$ holds with probability at least $1 - \delta$.
    
    From \cref{lemma:error propagation}, the setting that $\alpha = 1 - (1-\gamma)^2$, and the monotonicity of stochastic matrices,
    \begin{align}
        \vf{*} - \vf{\pi_K}
        &\leq
        2H \paren*{ \alpha^K + \frac{2 A_{\gamma, K}}{H} }\bone
        + \frac{1}{H^2} \underbrace{
            \paren*{\cN^{\pi_K} \pi_K + \cN^{\pi_*} \pi_*} \abs{E_K}
        }_{\df \heartsuit}
        \\
        &\hspace{1em}
        + \frac{1}{H^2} \underbrace{
            \sum_{k=1}^K \gamma^k \paren*{
                \cN^{\pi_*} \pi_* P_{K+1-k}^K + \cN^{\pi_K} \pi_K P_{K-k}^{K-1}
            } \paren*{ \abs{\varepsilon_{K+1-k}} + \frac{1}{H^2} \abs{E_{K-k}}}
        }_{\df \clubsuit}\,,
    \end{align}
    where $E_0 \df \bzero$.
    The first term can be bounded by $\square \alpha^K H \leq \square \varepsilon / c_3$ from \cref{lemma:A_gamma_k bound,lemma:k gamma to k-th inequality}.
    In the sequel, we derive upper bounds for $\heartsuit$ and $\clubsuit$.
    We note that $A_\infty = H^2$ and $A_{\gamma, k} \leq \alpha^k H$.

    Next, we derive an upper bound for $\heartsuit$.
    From the settings of $\alpha (\geq \gamma)$ and $M$,
    \begin{align}
        &2 V_k \iota_2
        \leq
        \frac{\square H \PVar (\vf{*}) \varepsilon^2}{c_4}
        + \frac{\square H \varepsilon^2}{c_4} \paren[\Bigg]{
            k \alpha^{2(k-2)}
            + \frac{H^3 \varepsilon^2}{c_4}
        } \bone\,.
    \end{align}
    From this result and \cref{lemma:sqrt inequality}, it follows that
    \begin{align}
        \frac{\heartsuit}{H^2}
        \leq
        \frac{\square \varepsilon^2}{c_4 H}
        + \frac{\square \varepsilon}{H \sqrt{c_4 H}} \paren*{\cN^{\pi_K} \pi_K + \cN^{\pi_*} \pi_*} \sigma (\vf{*})
        + \frac{\square \varepsilon}{\sqrt{c_4}} \paren[\bigg]{
            \underbrace{
                \sqrt{K/H} \alpha^{K-2}
            }_{\leq \varepsilon / c_3 \text{ from (a)}}
            + \underbrace{
                H \varepsilon / \sqrt{c_4}
            }_{\leq 1 / \sqrt{c_4} \text{ from (b)}}
        } \bone\,,
    \end{align}
    where (a) follows from \cref{lemma:k gamma to k-th inequality},
    and (b) follows by the assumption that $\varepsilon \leq 1/H$.
    By \cref{lemma:total variance},
    $\cN^{\pi_*} \pi_* \sigma (\vf{*}) \leq \square \sqrt{H^3}$.
    Furthermore, from \cref{lemma:coarse bound for last policy,lemma:popoviciu},
    \begin{align}
        \cN^{\pi_K} \pi_K \sigma (\vf{*})
        \leq
        \underbrace{\square H^2 \alpha^K}_{\leq \square H \sqrt{H} \text{ from (a)}}
        + \square \varepsilon H \sqrt{H / c_4}
        + \underbrace{
            \cN^{\pi_K} \pi_K \sigma (\vf{\pi_K})
        }_{\square H \sqrt{H} \text{ from \cref{lemma:total variance}}}
        \leq
        \square H \sqrt{H} \bone\,,
    \end{align}
    where (a) follows from \cref{lemma:k gamma to k-th inequality},
    and the last inequality follows since $\varepsilon \leq 1$.
    Consequently,
    $
        H^{-2} \heartsuit
        \leq
        \square \paren*{
            1/c_4 + 1/\sqrt{c_4}
        } \varepsilon \bone
    $.
    
    As for an upper bound for $\clubsuit$, we derive upper bounds for the following two components:
    \begin{gather}
        \diamondsuit
        \df 
        \frac{1}{H^2} \sum_{k=1}^{K-1} \gamma^k \cN^{\pi_*} \pi_* P_{K+1-k}^K \abs{E_{K-k}}
        \text{ and }
        \spadesuit
        \df
        \sum_{k=1}^K \gamma^k \cN^{\pi_*} \pi_* P_{K+1-k}^K \abs{\varepsilon_{K+1-k}}\,.
    \end{gather}
    Upper bounds for
    $
    H^{-2} \sum_{k=1}^K \gamma^k \cN^{\pi_K} \pi_K P_{K-k}^{K-1} \abs{E_{K-k}}
    $
    and
    $
    \sum_{k=1}^K \gamma^k \cN^{\pi_K} \pi_K P_{K-k}^{K-1} \abs{\varepsilon_{K+1-k}}
    $
    can be similarly derived.
    
    From \cref{lemma:E_k bound},
    $
        \diamondsuit
        \leq 
        \max_{k \in [K]} \infnorm{E_j} \bone
        \leq
        \square \varepsilon \sqrt{H^3 / c_4}
    $, and thus, $H^{-2} \diamondsuit \leq \square \varepsilon / \sqrt{c_4}$.
    On the other hand, from the assumption that $\gamma \leq \alpha$,
    \begin{align}
        2 W_k \iota_2
        &\leq
        \frac{\square \varepsilon^2}{c_4 H} \PVar(\vf{*})
        + \frac{\square H \varepsilon^2}{c_4} \paren[\Bigg]{
            \alpha^{2(k-2)}
            + \frac{\varepsilon^2 H}{c_4}
        } \bone
    \end{align}
    for $k > 1$.
    Using \cref{lemma:sqrt inequality,lemma:refined eps_k bound} as well as $\gamma \leq \alpha$,
    \begin{align}
        \spadesuit
        &\leq
        \square \varepsilon \cN^{\pi_*} \pi_* \sum_{k=1}^K \gamma^k P_{K+1-k}^K \paren*{
            \frac{\varepsilon}{c_4} \bone
            + \frac{\sigma(\vf{*})}{\sqrt{c_4 H}}
            + \sqrt{ \frac{H}{c_4} } \paren[\Bigg]{
                \alpha^{K-k-2} + \varepsilon \sqrt{ \frac{H}{c_4} }
            } \bone
        }
        \\
        &\leq
        \square \varepsilon \paren[\Bigg]{
            \frac{H^2 \varepsilon}{c_4} \bone
            + \cN^{\pi_*} \pi_* \sum_{k=1}^K \gamma^k P_{K+1-k}^K \frac{\sigma(\vf{*})}{\sqrt{c_4 H}}
            + \sqrt{ \frac{H^3}{c_4} } \paren[\Bigg]{
                \underbrace{
                    K \alpha^{K-2}
                }_{
                    \leq \square \varepsilon / c_3 \text{ from (a)}
                }
                + H \varepsilon \sqrt{ \frac{H}{c_4} }
            } \bone
        }
        \\
        &\leq
        \square \varepsilon \paren[\Bigg]{
            \underbrace{
                \frac{H^2 \varepsilon}{c_4} \bone
                + \sqrt{ \frac{H^3}{c_4} } \paren[\Bigg]{
                    \frac{\varepsilon}{c_3}
                    + H \varepsilon \sqrt{ \frac{H}{c_4} }
                } \bone
            }_{\leq H^2 / \sqrt{c_4} \text{ as } \varepsilon \leq 1/H}
        }
        + \frac{\square \varepsilon}{\sqrt{c_4 H}} \cN^{\pi_*} \pi_* \sum_{k=1}^K \gamma^k P_{K+1-k}^K \sigma(\vf{*})\,.
    \end{align}
    Now, it remains to upper-bound
    $\sum_{k=1}^K \gamma^k P_{K+1-k}^K \sigma(\vf{*})$.
    From \cref{lemma:variance decomposition},
    \begin{align}
        \sigma (\vf{*})
        \leq
        \sigma (\vf{*} - \vf{\pi_k'}) + \sigma (\vf{\pi_k'})
        \leq
        \square \alpha^k H \bone
        + \square \varepsilon \sqrt{H / c_4} \bone
        + \sigma (\vf{\pi_k'})
    \end{align}
    for any $k \in [K]$, where \cref{lemma:popoviciu,lemma:coarse bound for last policy} are used.
    Consequently,
    \begin{align}
        \sum_{k=1}^K \gamma^k P_{K+1-k}^K \sigma(\vf{*})
        &\leq
        \square \sum_{k=1}^K \gamma^k P_{K+1-k}^K \paren*{
            H \alpha^{K+1-k} \bone
            + \varepsilon \sqrt{H/c_4} \bone
            + \sigma(\vf{\pi_{K+1-k}'})
        }
        \\
        &\leq
        \square \paren[\Bigg]{
            \underbrace{H K \alpha^{K+1}}_{\leq \varepsilon / c_3} \bone
            + \varepsilon \sqrt{H^3 / c_4} \bone
            + \underbrace{
                \sum_{k=1}^K \gamma^k P_{K+1-k}^K \sigma(\vf{\pi_{K+1-k}'})
            }_{\leq \square \sqrt{H^3} \bone}
        }\,,
    \end{align}
    where the second inequality follows since $\gamma \leq \alpha$.
    Consequently,
    $
        H^{-2} \spadesuit
        \leq
        \square \varepsilon / \sqrt{c_4}
    $.    
    
    Combining these inequalities, we deduce that
    $
        \vf{*} - \vf{\pi_K} \leq \square \varepsilon \paren*{ c_3^{-1} + c_4^{-0.5}} \bone
    $.
\end{proof}

%% file: section/illustrations.tex
\section{Empirical illustration }
\label{sec:illustrations}

\begin{wrapfigure}{r}{0.48\textwidth}
    \vspace{-25pt}
    \centering
    \includegraphics[width=0.48\textwidth]{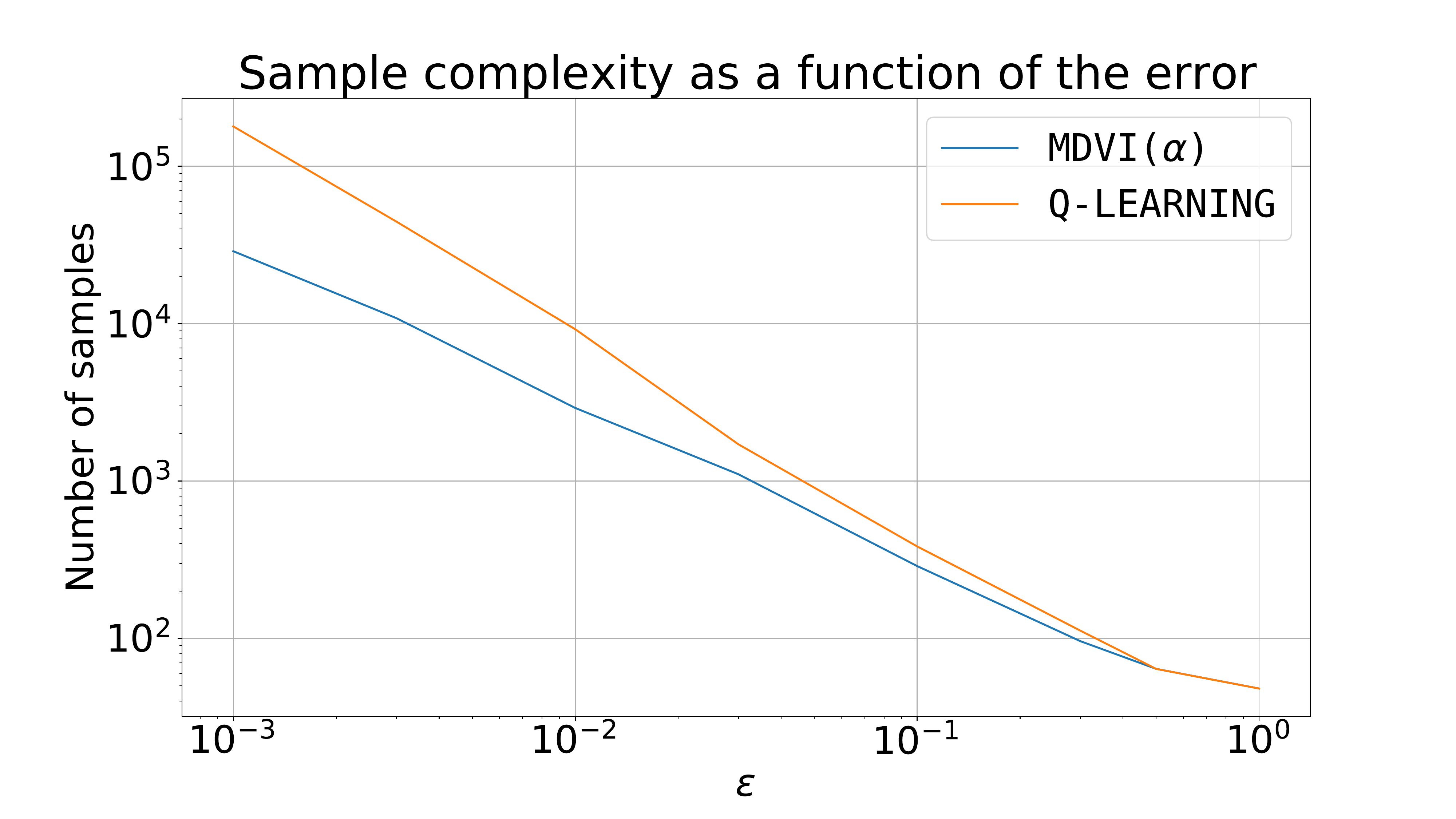}
    \caption{Sample complexities of \mdvi with $\alpha=1$ and \qlearning (synchronous version of Q-learning) on Garnets. \mdvi is run in the stationary policy setting. Both algorithms use $M=1$. As noted in \cref{sec:main result}, \mdvi with $\alpha=1$ is also nearly minimax-optimal.}
    \label{fig:sample_complexity}
\end{wrapfigure}

We compare \mdvi to a synchronous version of Q-learning (e.g., \citet{even2003learning}) in a simple setting on a class of random MDPs called Garnets~\citep{archibald1995generation}, with $\gamma=0.9$. \Cref{fig:sample_complexity} shows the sample complexity of \mdvi as a function of $\varepsilon$. We run \mdvi on 100 random MDPs, and, given $\varepsilon$, we report the number of samples $KM$ \mdvi uses to find $\varepsilon$-optimal policy. We compare this empirical sample complexity with the one of \qlearning, which has a tight quadratic dependency to the horizon \citep{li2021q} -- compared to the cubic one of \mdvi (\cref{theorem:pac bound}). \Cref{fig:sample_complexity} shows the difference in sample complexity between the two methods: especially for low $\varepsilon$, \mdvi reaches an $\varepsilon$-optimal policy with much fewer samples, up to $H=10$ times less samples for $\varepsilon=10^{-3}$. Complete details, pseudocodes, and results with other $\alpha$ are provided in \cref{app:illustrations}.

%% file: section/conclusion.tex
\section{Conclusion}

In this work, we considered and analyzed the sample complexity of
a model-free algorithm called MDVI \citep{geist2019theory,vieillard2020leverage}
under the generative model setting.
We showed that it is nearly minimax-optimal for finding
an $\varepsilon$-optimal policy despite its simplicity compared to
previous model-free algorithms \citep{sidford2018nearOptimal,wainwright2019variance,khamaru2021instance}.
We believe that our results are significant for the following three reasons.

First, we demonstrate the effectiveness of KL and entropy regularization.
Second, as discussed by \citet{vieillard2020leverage}, MDVI encompasses
various algorithms as special cases or equivalent forms, and
our results provide theoretical guarantees for most of them at once.
Third, MDVI uses no variance-reduction technique,
which leads to multi-epoch algorithms
and involved analyses \citep{sidford2018nearOptimal,wainwright2019variance,khamaru2021instance}.
As such, our analysis is straightforward, and it would be easy
to extend it to more complex settings.

A disadvantage of \mdvi is that its range of valid $\varepsilon$ is limited
compared to previous algorithms \citep{sidford2018nearOptimal,agarwal2020modeBased,li2020breaking}.
It is unclear if this is an artifact of our analysis or the real limitation of
MDVI-type algorithms.
We leave this topic as a future work.

%% file: appendix/notations.tex
\section{Notations}
\label{app:notations}

\begin{table}[h]
	\centering
	\caption{Table of Notations}
	\begin{tabular}{@{}l|l@{}}
		\toprule
		\thead{Notation} & \thead{Meaning} \\ \midrule
	$\A$ & action space of size $A$\\
	$H$ & effective horizon $H \df 1 / (1 - \gamma)$\\
	$P$ & transition matrix\\
	$\X$ & state space of size $X$\\
	$r$ & reward vector bounded by $1$\\
	$\gamma$ & discount factor in $[0, 1)$\\
	\midrule
	$\varepsilon$ & admissible suboptimality\\
	$\delta$ & admissible failure probability\\
	\midrule
	$E_k$ & $E_k : (x, a) \mapsto \sum_{j=1}^k \alpha^{k-j} \varepsilon_j (x, a)$\\
	$\varepsilon_k$ & $\varepsilon_k: (x, a) \mapsto \gamma \widehat{P}_{k-1} v_{k-1} (x, a) - \gamma P v_{k-1} (x, a)$\\
	\midrule
	$A_{k}, A_\infty, A_{\gamma, k}$ & $\sum_{j=0}^{k-1} \alpha^j$, $\sum_{j=0}^{\infty} \alpha^j$, $\sum_{j=0}^{k-1} \alpha^j \gamma^{k-j}$\\
	$\Eone$ & event of small $E_k$ for all $k$ (not variance-aware)\\
	$\Etwo$ & event of small $\varepsilon_k$ for all $k$ (not variance-aware)\\
	$\Ethree$ & event of small $E_k$ for all $k$ (variance-aware)\\
	$\Efour$ & event of small $\varepsilon_k$ for all $k$ (variance-aware)\\
	$\bF_{k, m}$ & $\sigma$-algebra in the filtration (cf.~\cref{sec:main result})\\
	$K$ & number of value updates \\
	$M$ & number of samples per each value update\\
	$P^\pi$ & $P^\pi \df P \pi$\\
	$P^i_j$, $P_*^i$ & $P^i_j \df P^{\pi_i} P^{\pi_{i-1}} \cdots P^{\pi_{j+1}} P^{\pi_j}$, $P_*^i \df (P^{\pi_*})^i$\\
	$T^{\pi}$, $T^i_j$ & Bellman operator for a policy $\pi$, $T^i_j \df T^{\pi_i} T^{\pi_{i-1}} \cdots T^{\pi_{j+1}} T^{\pi_j}$\\
	$V_k$ & an upper bound for $E_k$'s predictive quadratic variance (cf.~\cref{lemma:refined E_k bound}) \\
	$W_k$ & an upper bound for $\varepsilon_k$'s predictive quadratic variance (cf.~\cref{lemma:refined eps_k bound}) \\
	$s_k$ & $s_k \df q_k + \alpha s_{k-1}$ (cf.~\mdvi) \\
	$v_k$ & $v_k \df w_k - \alpha w_{k-1}$ (cf.~\mdvi) \\
    $w_k$ & $w_k (x) \df \max_{a \in \A} s_k (x, a)$ (cf.~\mdvi) \\
    $\alpha$ & $\alpha \df \tau / (\tau + \kappa)$, weight for $s_k$ updates (cf.~\mdvi and \cref{sec:equivalence proof})\\
	$\beta$ & $\beta \df 1 / (\tau + \kappa)$, inverse temperature for $\pi_k$ (cf.~\cref{sec:mdvi} and \cref{sec:equivalence proof})\\
	$\iota_1$, $\iota_2$ & $\iota_1 \df \log (8 K \aXA / \delta)$, $\iota_2 \df \log (16 K \aXA / \delta)$\\
	$\pi_k'$ & a non-stationary policy that follows $\pi_k, \pi_{k-1}, \ldots$ sequentially (cf.~\cref{sec:main result})\\
	$\square$ & an indefinite constant independent of $H$, $\aX$, $\aA$, $\varepsilon$, and $\delta$ \\
	\bottomrule
	\end{tabular}
\end{table}

%% file: appendix/equivalence_proof.tex
\section{Equivalence of MDVI Update Rules}\label{sec:equivalence proof}

We show the equivalence of MDVI's
updates \eqref{eq:MDVI update} and \eqref{eq:policy update} to
those used in \mdvi.
We first recall MDVI's updates \eqref{eq:MDVI update} and \eqref{eq:policy update}:
\begin{gather}
    q_{k+1} = r + \gamma P^{\pi_k} \paren*{
        q_k - \tau \log \frac{\pi_k}{\pi_{k-1}} - \kappa \log \pi_k
    } + \varepsilon_k\,,
\end{gather}
\begin{align}
    \text{where }
    \pi_k \parenc*{\cdot}{x}
    =
    \argmax_{p \in \Delta(\A)}
    \sum_{a \in \A} p (a) \paren*{
        q_k (s, a) - \tau \log \frac{p (a)}{\pi_{k-1} \parenc*{a}{x}} - \kappa \log p (a)
    }
    \text{ for all } x \in \X\,,
\end{align}

The policy update \eqref{eq:policy update} can be rewritten as follows (e.g., Equation~(5) of \citet{kozuno2019theoretical}):
\begin{gather}
    \pi_k (a|x)
    =
    \frac{
        \pi_{k-1} (a|x)^\alpha \exp \paren*{ \beta q_k (x, a) }
    }{
        \sum_{b \in \A} \pi_{k-1} (b|x)^\alpha \exp \paren*{ \beta q_k (x, b) }
    }\,,
\end{gather}
where $\alpha \df \tau / (\tau + \kappa)$, and $\beta \df 1 / (\tau + \kappa)$.
It can be further rewritten as, defining $s_k = q_k + \alpha s_{k-1}$
\begin{gather}
    \pi_k (a|x)
    =
    \frac{
        \exp \paren*{ \beta s_k (x, a) }
    }{
        \sum_{b \in \A} \exp \paren*{ \beta s_k (x, b) }
    }\,.
\end{gather}
Plugging in this policy expression to $v_k$, we deduce that
\begin{align}
    v_k (x)
    &=
    \frac{1}{\beta} \log \sum_{a \in \A} \exp \paren*{ \beta q_k (x, a) + \alpha\log \pi_{k-1} (a|x) }
    \\
    &=
    \frac{1}{\beta} \log \sum_{a \in \A} \exp \paren*{ \beta s_k (x, a) }
    -
    \frac{\alpha}{\beta} \log \sum_{a \in \A} \exp \paren*{ \beta s_{k-1} (x, a) }\,.
\end{align}
\citet[Appendix~B]{kozuno2019theoretical} show that when $\beta \to \infty$,
$
    v_k (x)
    =
    w_k (x) - \alpha w_{k-1} (x)\,.
$
Furthermore, the Boltzmann policy becomes a greedy policy.
Accordingly, the update rules used in \mdvi is a limit case of the original MDVI updates.

%% file: appendix/auxiliary_lemmas.tex
\section{Auxiliary Lemmas}

In this appendix, we prove some auxiliary lemmas used in the proof.

\begin{lemma}\label{lemma:sqrt inequality}
    For any positive real values $a$ and $b$, $\sqrt{a + b} \leq \sqrt{a} + \sqrt{b}$.
\end{lemma}

\begin{proof}
    Indeed, $a + b \leq a + 2 \sqrt{ab} + b = (\sqrt{a} + \sqrt{b})^2$.
\end{proof}

\begin{lemma}\label{lemma:square inequality}
    For any real values $(a_n)_{n=1}^N$, $(\sum_{n=1}^N a_n)^2 \leq N \sum_{n=1}^N a_n^2$.
\end{lemma}

\begin{proof}
    Indeed, from the Cauchy–Schwarz inequality,
    \begin{align}
        \paren*{ \sum_{n=1}^N a_n \cdot 1 }^2
        \leq
        \paren*{ \sum_{n=1}^N 1 } \paren*{ \sum_{n=1}^N a_n^2 }
        =
        N \sum_{n=1}^N a_n^2\,,
    \end{align}
    which is the desired result.
\end{proof}

\begin{lemma}\label{lemma:A_gamma_k bound}
    For any $k \in [K]$,
    \begin{align}
        A_{\gamma, k}
        =
        \begin{cases}
            \gamma \dfrac{\alpha^k - \gamma^k}{\alpha - \gamma} & \text{if } \alpha \neq \gamma
            \\
            k \gamma^k & \text{otherwise}
        \end{cases}\,.
    \end{align}
\end{lemma}

\begin{proof}
    Indeed, if $\alpha \neq \gamma$
    \begin{align}
        A_{\gamma, k}
        =
        \sum_{j=0}^{k-1} \alpha^j \gamma^{k-j}
        =
        \gamma^k \frac{(\alpha / \gamma)^k - 1}{(\alpha / \gamma) - 1}
        =
        \gamma \frac{\alpha^k - \gamma^k}{\alpha - \gamma}\,.
    \end{align}
    If $\alpha = \gamma$, $A_{\gamma, k} = k \gamma^k$ by definition.
\end{proof}

\begin{lemma}\label{lemma:log reciprocal inequality}
    For any real value $x \in (0, 1]$, $1 - x \leq \log (1/x)$.
\end{lemma}

\begin{proof}
    Since $\log (1/x)$ is convex and differentiable, $\log (1/x) \geq \log (1/y) - (x - y) / y$. Choosing $y=1$, we concludes the proof.
\end{proof}

\begin{lemma}\label{lemma:k gamma to k-th inequality}
    Suppose $\alpha, \gamma \in [0, 1)$, $\varepsilon \in (0, 1]$, $c \in [1, \infty)$, $m \in \N$, and $n \in [0, \infty)$.
    Let $K \df \dfrac{m}{1-\alpha} \log \dfrac{c H}{\varepsilon}$.
    Then,
    \begin{align}
        K^n \alpha^K
        \leq
        \paren*{\frac{mn}{(1 - \alpha)e}}^n
        \paren*{\dfrac{\varepsilon}{c H}}^{m-1}\,.
    \end{align}
\end{lemma}

\begin{proof}
    Using \cref{lemma:log reciprocal inequality},
    \begin{align}
        K
        = \frac{m}{1-\alpha} \log \dfrac{c H}{\varepsilon}
        \geq \log_\alpha \paren*{\dfrac{\varepsilon}{c H}}^m.
    \end{align}
    Therefore,
    \begin{align}
        K^n \alpha^K
        \leq
        \paren*{\frac{m}{1 - \alpha} \log \dfrac{c H}{\varepsilon}}^n \paren*{\dfrac{\varepsilon}{c H}}^m
        =
        \frac{m^n}{(1 - \alpha)^n}
        \paren*{\dfrac{\varepsilon}{c H}}^m
        \paren*{\log \dfrac{c H}{\varepsilon}}^n\,.
    \end{align}
    Since $x \paren*{\log \dfrac{1}{x}}^n \leq \paren*{\dfrac{n}{e}}^n$ for any $x \in (0, 1]$ as shown later,
    \begin{align}
        K^n \alpha^K
        \leq
        \paren*{\frac{mn}{(1 - \alpha)e}}^n
        \paren*{\dfrac{\varepsilon}{c H}}^{m-1}\,.
    \end{align}
    
    Now it remains to show $f (x) \df x \paren*{\log \dfrac{1}{x}}^n \leq \paren*{\dfrac{n}{e}}^n$ for $x < 1$.
    We have that
    \begin{align}
        f'(x) = (- \log x)^n - n (- \log x)^{n-1}
        \implies
        f'(x) = 0
        \text{ at }
        x = e^{-n}.
    \end{align}
    Therefore, $f$ takes its maximum $\paren*{\dfrac{n}{e}}^n$ at $e^{-n}$ when $x \in (0, 1)$.
\end{proof}

The following lemma is a special case of a well-known inequality that
for any increasing function $f$
\begin{align}
    \sum_{k=1}^K f (k) \leq \int_1^{K+1} f(x) dx\,.
\end{align}

\begin{lemma}\label{lemma:sum of k power}
    For any $K \in \N$ and $n \in [0, \infty)$,
    $
        \displaystyle \sum_{k=1}^K k^n \leq \frac{1}{n+1} (K+1)^{n+1}
    $.
\end{lemma}

%% file: appendix/tools_from_prob_theory.tex
\section{Tools from Probability Theory}

We extensively use the following two concentration inequalities. The first one is Azuma-Hoeffding inequality \citep{azuma1967weighted,hoeffding1963probability,boucheron2013concentration}, and the second one is Bernstein's inequality \citep{bernstein1946theory,boucheron2013concentration} for a martingale \citep[Excercises 5.14 (f)]{lattimore2020bandit}. For a real-valued stochastic process $(X_n)_{n=1}^N$ adapted to a filtration $(\cF_n)_{n=1}^N$, we let $\E_n [X_n] \df \E \brackc{X_n}{\cF_{n-1}}$ for $n \geq 1$, and $\E_1 [X_1] \df \E \brack{X_1}$.

\begin{lemma}[Azuma-Hoeffding Inequality]\label{lemma:hoeffding}
    Consider a real-valued stochastic process $(X_n)_{n=1}^N$ adapted to a filtration $(\cF_n)_{n=1}^N$.
    Assume that $X_n \in [l_n, u_n]$ and $\E_n [X_n] = 0$ almost surely, for all $n$. Then,
	\begin{align}
	\P \paren*{
      \sum_{n=1}^N X_n
      \geq
      \sqrt{
        \sum_{n=1}^N \frac{(u_n - l_n)^2}{2} \log \frac{1}{\delta}
      }
    }
    \leq \delta
	\end{align}
	for any $\delta \in (0, 1)$.
\end{lemma}

\begin{lemma}[Bernstein's Inequality]\label{lemma:bernstein}
	Consider a real-valued stochastic process $(X_n)_{n=1}^N$ adapted to a filtration $(\cF_n)_{n=1}^N$.
    Suppose that $X_n \leq U$ and $\E_n [X_n] = 0$ almost surely, for all $n$.
    Then, letting $V' \df \sum_{n=1}^N \E_n[X_n^2]$,
	\begin{align}
        \P \paren*{
            \sum_{n=1}^N X_n
            \geq
            \frac{2 U}{3} \log \frac{1}{\delta}
            + \sqrt{ 2 V \log \frac{1}{\delta} }
            \text{ and }
            V' \leq V
        } \leq \delta
	\end{align}
	for any $V \in [0, \infty)$ and $\delta \in (0, 1)$.
\end{lemma}

In our analysis, we use the following corollary of this Bernstein's inequality.

\begin{lemma}[Conditional Bernstein's Inequality]\label{lemma:conditional bernstein}
    Consider the same notations and assumptions in \cref{lemma:bernstein}.
    Furthermore, let $\cE$ be an event that implies $V' \leq V$ for some $V \in [0, \infty)$ with $\P (\cE) \geq 1 - \delta'$ for some $\delta' \in (0, 1)$.
    Then,
	\begin{align}
        \P \parenc*{
            \sum_{n=1}^N X_n
            \geq
            \frac{2 U}{3} \log \frac{1}{\delta (1-\delta')}
            + \sqrt{ 2 V \log \frac{1}{\delta (1-\delta')} }
        }{
            \cE
        } \leq \delta
	\end{align}
	for any $\delta \in (0, 1)$.
\end{lemma}

\begin{proof}
    Let $A$ and $B$ denote the events of
    \begin{align}
        \sum_{n=1}^N X_n
        \geq
        \frac{2 U}{3} \log \frac{1}{\delta (1-\delta')}
        + \sqrt{ 2 V \log \frac{1}{\delta (1-\delta')} }
    \end{align}
    and $V' \leq V$, respectively.
    Since $\cE \subset B$, it follows that $A \cap \cE \subset A \cap B$, and
    $
        \P (A \cap \cE) \leq \P (A \cap B)
    $.
    Accordingly,
    \begin{align}
        \P (A | \cE)
        = \frac{\P (A \cap \cE)}{\P (\cE)}
        \leq \frac{\P (A \cap B)}{\P (\cE)}
        \numeq{\leq}{a} \frac{\delta (1-\delta')}{\P (\cE)}
        \numeq{\leq}{b} \delta\,,
    \end{align}
    where (a) follows from \cref{lemma:bernstein}, and (b) follows from $\P(\cE) \geq 1-\delta'$.
\end{proof}

\begin{lemma}[Popoviciu's Inequality for Variances]\label{lemma:popoviciu}
    The variance of any random variable bounded by $x$ is bounded by $x^2$.
\end{lemma}

%% file: appendix/total_variance.tex
\section{Total Variance Technique}

The following lemma is due to \citet{azar2013minimax}.

\begin{lemma}\label{lemma:variance decomposition}
    Suppose two real-valued random variables $X, Y$ whose variances, $\Var X$ and $\Var Y$, exist and are finite. Then, $\sqrt{\Var X} \leq \sqrt{\Var \brack*{X - Y}} + \sqrt{\Var Y}$.
\end{lemma}

For completeness, we prove \cref{lemma:variance decomposition}.

\begin{proof}
  Indeed, from Cauchy-Schwartz inequality,
  \begin{align}
    \Var X
    &=
    \Var \brack{X - Y + Y}
    \\
    &=
    \Var \brack{X - Y}
    + \Var Y
    + 2 \E \brack*{ (X - Y - \E \brack{X-Y} ) (Y - \E Y) }
    \\
    &\leq
    \Var \brack{X - Y}
    + \Var Y
    + 2 \sqrt{
      \Var \brack{X - Y} \Var Y
    }
    =
    \paren*{
      \sqrt{\Var \brack*{X - Y}} + \sqrt{\Var Y}
    }^2\,.
  \end{align}
  This is the desired result.
\end{proof}

The following lemma is an extension of Lemma~7 by \citet{azar2013minimax} and its refined version by \citet{agarwal2020modeBased}.

\begin{lemma}\label{lemma:total variance}
    Suppose a sequence of deterministic policies $( \pi_k )_{k=0}^K$
    and let
    \begin{align}
        \qf{\pi'_k}
        \df
        \begin{cases}
            r + \gamma P \vf{\pi'_{k-1}} & \text{for } k \in [K]
            \\
            \qf{\pi_0} & \text{for } k = 0
        \end{cases}\,.
    \end{align}
    Furthermore, let $\sigma_k^2$ and $\Sigma_k^2$ be non-negative functions over $\XA$ defined by
    \begin{align}
        \sigma_k^2 (x, a)
        \df
        \begin{cases}
            P \paren{ \vf{\pi'_{k-1}} }^2 (x, a) - \paren{P \vf{\pi'_{k-1}} }^2 (x, a) & \text{for } k \in [K]
            \\
            P \paren{ \vf{\pi_0} }^2 (x, a) - \paren{P \vf{\pi_0} }^2 (x, a) & \text{for } k = 0
        \end{cases}
    \end{align}
    and
    \begin{gather}
        \Sigma_k^2 (x, a)
        \df
        \E_k \brackc*{
            \paren*{ \sum_{t=0}^\infty \gamma^t r (X_t, A_t) - \qf{\pi'_k} (X_0, A_0) }^2
        }{X_0=x, A_0=a}
    \end{gather}
    for $k \in \{0\} \cup [K]$, where $\E_k$ is the expectation over $(X_t, A_t)_{t=0}^\infty$ wherein $A_t \sim \pi_{k-t} (\cdot | X_t)$
    until $t = k$, and $A_t \sim \pi_0 (\cdot | X_t)$ thereafter.
    Then,
    \begin{align}
        \sum_{j=0}^{k-1} \gamma^{j + 1} P_{k-j}^{k-1} \sigma_{k-j} \leq \sqrt{2 H^3}
    \end{align}
    for any $k \in [K]$.
\end{lemma}

For its proof, we need the following lemma.

\begin{lemma}\label{lemma:variance Bellman eq for non-stationary policy}
    Suppose a sequence of deterministic policies $( \pi_k )_{k = 0}^K$ and notations in \cref{lemma:total variance}.
    Then, for any $k \in [K]$, we have that
    \begin{align}
        \Sigma_k^2
        =
        \gamma^2 \sigma_k^2 + \gamma^2 P^{\pi_{k-1}} \Sigma_{k-1}^2\,.
    \end{align}
\end{lemma}

\begin{proof}
    Let $R_s^u \df \sum_{t=s}^u \gamma^{t-s} r (X_t, A_t)$
    and $\E_k \brackc*{\cdot}{x, a} \df \E_k \brackc*{\cdot}{X_0=x, A_0=a}$.
    We have that
    \begin{align}
        \Sigma_k^2 (x, a)
        =
        \E_k \brackc*{
          \paren*{ R_0^\infty - \qf{\pi'_k} (X_0, A_0) }^2
        }{x, a}
        \df
        \E_k \brackc*{
          \paren*{ I_1 + \gamma I_2 }^2
        }{x, a}\,,
    \end{align}
    where $I_1 \df r (X_0, A_0) + \gamma \qf{\pi'_{k-1}} (X_1, A_1) - \qf{\pi'_k} (X_0, A_0)$, and $I_2 \df R_1^\infty - \qf{\pi'_{k-1}} (X_1, A_1)$.
    With these notations, we see that
    \begin{align}
    \Sigma_k^2 (x, a)
    &=
    \E_k \brackc[\big]{
      I_1^2 + \gamma^2 I_2^2 + 2 \gamma I_1 I_2
    }{x, a}
    \\
    &=
    \E_k \brackc[\big]{
      I_1^2
      + \gamma^2 I_2^2
      + 2 \gamma I_1 \E_{k-1} \brackc*{I_2}{X_1, A_1}
    }{x, a}
    \\
    &=
    \E_k \brackc[\big]{I_1^2}{x, a} + \gamma^2 \E_k \brackc[\big]{I_2^2}{x, a}
    \\
    &=
    \E_k \brackc[\big]{I_1^2}{x, a}
    + \gamma^2 P^{\pi_{k-1}} \Sigma_{k-1}^2 (x, a)\,,
    \end{align}
    where the second line follows from the law of total expectation,
    and the third line follows since
    $\E_{k-1} \brackc*{I_2}{X_1, A_1} = 0$ due to the Markov property.
    The first term in the last line is $\gamma^2 \sigma_k^2 (x, a)$
    because
    \begin{align}
        \E_k \brackc[\big]{I_1^2}{x, a}
        &\numeq{=}{a}
        \gamma^2 \E_k \brackc[\Bigg]{
            \paren[\Big]{
                \underbrace{\qf{\pi_{k-1}'} (X_1, A_1)}_{
                    \vf{\pi_{k-1}'} (X_1) \text{ from (b)}
                }
                - (P \vf{\pi_{k-1}'}) (X_0, A_0)
            }^2
        }{x, a}
        \\
        &=
        \gamma^2 \paren*{ P \paren*{\vf{\pi_{k-1}'}}^2} (x, a)
        + \gamma^2 (P \vf{\pi_{k-1}'})^2 (x, a)
        - 2 (P \vf{\pi_{k-1}'})^2 (x, a)
        \\
        &=
        \gamma^2 \paren*{ P \paren*{\vf{\pi_{k-1}'}}^2} (x, a)
        - \gamma^2 (P \vf{\pi_{k-1}'})^2 (x, a)\,,
    \end{align}
    where (a) follows from the definition that $\qf{\pi_k'} = r + \gamma P \vf{\pi'_{k-1}}$,
    and (b) follows since the policies are deterministic.
    From this argument, it is clear that
    $
        \Sigma_k^2
        =
        \gamma^2 \sigma_k^2 + \gamma^2 P^{\pi_{k-1}} \Sigma_{k-1}^2\,,
    $
    which is the desired result.
\end{proof}

Now, we are ready to prove \cref{lemma:total variance}.

\begin{proof}[Proof of \cref{lemma:total variance}]
  Let $H_k \df \sum_{j=0}^{k-1} \gamma^j$. Using Jensen's inequality twice,
  \begin{align}
    \sum_{j=0}^{k-1} \gamma^{j + 1} P_{k-j}^{k-1} \sigma_{k-j}
    &\leq
    \sum_{j=0}^{k-1} {
      \gamma^{j + 1}
      \sqrt{
        P_{k-j}^{k-1} \sigma_{k-j}^2
      }
    }
    \\
    &\leq
    \gamma H_k \sum_{j=0}^{k-1} {
      \frac{\gamma^{j + 1}}{H_k}
      \sqrt{
        P_{k-j}^{k-1} \sigma_{k-j}^2
      }
    }
    \\
    &\leq
    \sqrt{
      H_k
      \sum_{j=0}^{k-1} {
        \gamma^{j+2} P_{k-j}^{k-1} \sigma_{k-j}^2
      }
    }
    \leq
    \sqrt{
      H
      \sum_{j=0}^{k-1} {
        \gamma^{j+2} P_{k-j}^{k-1} \sigma_{k-j}^2
      }
    }\,.
  \end{align}
  From \cref{lemma:variance Bellman eq for non-stationary policy}, we have that
  \begin{align}
    &\hspace{-1em}\sum_{j=0}^{k-1} {
      \gamma^{j+2} P_{k-j}^{k-1} \sigma_{k-j}^2
    }
    \\
    &=
    \sum_{j=0}^{k-1} {
      \gamma^j P_{k-j}^{k-1} \paren*{
        \Sigma_{k-j}^2 - \gamma^2 P^{\pi_{k-1-j}} \Sigma_{k-1-j}^2
      }
    }
    \\
    &=
    \sum_{j=0}^{k-1} {
      \gamma^j P_{k-j}^{k-1} \paren*{
        \Sigma_{k-j}^2
        - \gamma P^{\pi_{k-1-j}} \Sigma_{k-1-j}^2
        + \gamma (1-\gamma) P^{\pi_{k-1-j}} \Sigma_{k-1-j}^2
      }
    }
    \\
    &=
    \sum_{j=0}^{k-1} \gamma^j P_{k-j}^{k-1} \Sigma_{k-j}^2
    - \sum_{j=1}^k \gamma^j P_{k-j}^{k-1} \Sigma_{k-j}^2
    + \gamma (1-\gamma) \sum_{j=0}^{k-1} \gamma^j P_{k-1-j}^{k-1} \Sigma_{k-1-j}^2\,.
  \end{align}
  The final line is equal to
  $
    \Sigma_k^2
    - \gamma^k P_0^{k-1} \Sigma_0^2
    + \gamma (1-\gamma) \sum_{j=0}^{k-1} \gamma^j P_{k-1-j}^{k-1} \Sigma_{k-1-j}^2
  $.
  Finally, from the monotonicity of stochastic matrices
  and that $\bzero \leq \Sigma_j^2 \leq H^2 \bone$ for any $j$,
  \begin{align}
    \sum_{j=0}^{k-1} \gamma^{j + 1} P_{k-j}^{k-1} \sigma_{k-j} \leq \sqrt{2 H^3}\,.
  \end{align}
  This concludes the proof.
\end{proof}

%% file: appendix/lemma_proof.tex
\section{Proof of Lemmas for \texorpdfstring{
        \cref{theorem:non-stationary pac bound}
    }{
        Theorem~\ref{theorem:non-stationary pac bound}
    } (Bound for a Non-Stationary Policy)
}\label{sec:proof of lemmas for non-stationary policy}

Before starting the proof, we introduce some notations and facts frequently used in the proof.

\paragraph{Frequently Used Facts.}
We frequently use the following fact, which follows from definitions:
\begin{gather}
    s_k
    = A_k r + \gamma P w_{k-1} + E_k
    \quad
    \text{for any } k \in [K]\,.\label{eq:rewrite s_k}
\end{gather}
Indeed,
$
    s_k
    = \sum_{j=1}^k \alpha^{k-j} (r + \gamma P (w_{j-1} - \alpha w_{j-2}) + \varepsilon_j)
    = A_k r + \gamma P w_{k-1} + E_k
$.
In addition, we often mention the ``monotonicity'' of stochastic matrices:
any stochastic matrix $\rho$ satisfies that
$\rho v \geq \rho u$ for any vectors $v, u$ such that $v \geq u$.
Examples of stochastic matrices in the proof are $P$, $\pi$, $P^{\pi}$, and $\pi P$.
The monotonicity property is so frequently used that we do not always mention it.

\subsection{Proof of \texorpdfstring{
        \cref{lemma:non-stationary error propagation}
    }{
        Lemma~\ref{lemma:non-stationary error propagation}
    } (Error Propagation Analysis)
}\label{subsec:proof of non-stationary error propagation}
\input{proof/non_stationary_error_propagation}

\subsection{Proof of \texorpdfstring{
        \cref{lemma:E_k bound,lemma:non-stationary coarse bound}
    }{
        Lemmas~\ref{lemma:E_k bound} and \ref{lemma:non-stationary coarse bound}
    } (Coarse State-Value Bound)
}\label{subsec:proof of non-stationary coarse bound}

\input{proof/E_k_bound}
\input{proof/non_stationary_coarse_policy_bound}

\subsection{Proof of \texorpdfstring{
        \cref{lemma:v error prop}
    }{
        Lemma~\ref{lemma:v error prop}
    } (Value Estimation Error Bound)
}\label{subsec:proof of v error prop}
\input{proof/v_error_propagation}

\subsection{Proof of \texorpdfstring{
        \cref{lemma:variance upper bounds,lemma:eps_k bound}
    }{
        Lemmas~\ref{lemma:variance upper bounds} and \ref{lemma:eps_k bound}
    } (Value Estimation Variance Bound)
}\label{subsec:proof of coarse v bound}
\input{proof/coarse_v_bound}

\subsection{Proof of \texorpdfstring{
        \cref{lemma:refined eps_k bound,lemma:refined E_k bound}
    }{
        Lemmas~\ref{lemma:refined E_k bound} and \ref{lemma:refined eps_k bound}
    } (Error Bounds with Bernstein's Inequality)
}\label{subsec:proof of refined bound}
\input{proof/refined_eps_k_and_E_k_bound}

\section{Proof of Lemmas for \texorpdfstring{
        \cref{theorem:pac bound}
    }{
        Theorem~\ref{theorem:pac bound}
    } (Bound for a Stationary Policy)
}

We use the same notations as those used in \cref{sec:proof of lemmas for non-stationary policy}.

\subsection{Proof of \texorpdfstring{
        \cref{lemma:error propagation}
    }{
        Lemma~\ref{lemma:error propagation}
    } (Error Propagation Analysis)
}\label{subsec:proof of error propagation}
\input{proof/last_policy_error_propagation}

\subsection{Proof of \texorpdfstring{
        \cref{lemma:coarse bound for last policy}
    }{
        Lemma~\ref{lemma:coarse bound for last policy}
    } (Coarse State-Value Bounds)
}\label{subsec:proof of last policy coarse bound}
\input{proof/last_policy_coarse_policy_bound}

%% file: proof/non_stationary_error_propagation.tex
\begin{proof}
    Note that
    \begin{align}
        \bzero
        \leq
        \vf{*} - \vf{\pi'_k}
        =
        \frac{A_k}{A_\infty} \paren*{ \vf{*} - \vf{\pi'_k} }
        + \alpha^k \paren*{ \vf{*} - \vf{\pi'_k} }
        \leq
        \frac{A_k}{A_\infty} \paren*{ \vf{*} - \vf{\pi'_k} }
        + 2 H \alpha^k \bone
    \end{align}
    since $\vf{*} - \vf{\pi'_k} \leq 2 H \bone$.
    Therefore, we need an upper bound for $A_k (\vf{*} - \vf{\pi'_k})$.
    We decompose $A_k (\vf{*} - \vf{\pi'_k})$ to $A_k \vf{*} - w_k$ and $w_k - A_k\vf{\pi'_k}$.
    Then, we derive upper bounds for each of them
    (inequalities \eqref{eq:vstar - m s_k bound upper bound} and \eqref{eq:m s_k - vpi prime bound upper bound}, respectively).
    The desired result is obtained by summing up those bounds.
    
    \paragraph{Upper bound for $A_k \vf{*} - w_k$.}
    We prove by induction that for any $k \in [K]$,
    \begin{align}\label{eq:vstar - m s_k bound upper bound}
        A_k \vf{*} - w_k
        \leq
        H A_{\gamma, k} \bone - \sum_{j=0}^{k-1} \gamma^j \pi_* P_*^j E_{k-j}\,.
    \end{align}
    We have that
    \begin{align}
        A_k \vf{*} - w_k
        &\numeq{\leq}{a}
        \pi_* (A_k \qf{*} - s_k)
        \\
        &\numeq{=}{b}
        \pi_* \paren*{
            A_k \qf{*} - A_k r - \gamma P w_{k-1} - E_k
        }
        \\
        &\numeq{=}{c}
        \pi_* \paren*{
            \gamma P (A_k \vf{*} - w_{k-1}) - E_k
        }
        \\
        &\numeq{\leq}{d}
        \pi_* \paren*{
            \gamma P (A_{k-1} \vf{*} - w_{k-1})
            + \alpha^{k-1} \gamma H \bone
            - E_k
        }\,,
    \end{align}
    where (a) is due to the greediness of $\pi_k$,
    (b) is due to the equation \eqref{eq:rewrite s_k},
    (c) is due to the Bellman equation for $\qf{*}$,
    and (d) is due to the fact that $(A_k - A_{k-1}) \vf{*} = \alpha^{k-1} \vf{*} \leq \alpha^{k-1} H \bone$.
    From this result and the fact that $w_0 = \bzero$,
    $
        A_1 \vf{*} - w_1 \leq \gamma H \bone - \pi_* E_1\,.
    $
    Therefore, the inequality \eqref{eq:vstar - m s_k bound upper bound} holds for $k=1$.
    From the step (d) above and induction, it is straightforward to verify that the inequality \eqref{eq:vstar - m s_k bound upper bound} holds for other $k$.

    \paragraph{Upper bound for $w_k - A_k \vf{\pi'_k}$.}
    We prove by induction that for any $k \in [K]$,
    \begin{align}\label{eq:m s_k - vpi prime bound upper bound}
        w_k - A_k \vf{\pi'_k}
        \leq
        H A_{\gamma, k} \bone + \sum_{j=0}^{k-1} \gamma^j \pi_k P_{k-j}^{k-1} E_{k-j}\,.
    \end{align}
    Recalling that $\vf{\pi'_k} = \pi_k T_0^{k-1} \qf{\pi_0}$,
    we deduce that
    \begin{align}
        w_k - A_k \vf{\pi'_k}
        &\numeq{=}{a}
        \pi_k \paren*{ s_k - A_k T_0^{k-1} \qf{\pi_0} }
        \\
        &\numeq{=}{b}
        \pi_k \paren*{
            A_k r + \gamma P w_{k-1}
            - A_k T_1^{k-1} \qf{\pi_0}
            + E_k
        }
        \\
        &\numeq{=}{c}
        \pi_k \paren*{
            \gamma P \paren*{w_{k-1} - A_k \vf{\pi'_{k-1}}} + E_k
        }
        \\
        &\numeq{\leq}{d}
        \pi_k \paren*{
            \gamma P (w_{k-1} - A_{k-1} \vf{\pi'_{k-1}})
            + \alpha^{k-1} \gamma H \bone
            + E_k
        }\,,
    \end{align}
    where (a) follows from the definition of $w_k$,
    (b) is due to the equation \eqref{eq:rewrite s_k},
    (c) follows from the definition of the Bellman operator,
    and (d) is due to the fact that $(A_k - A_{k-1}) \vf{\pi'_{k-1}} = \alpha^{k-1} \vf{\pi'_{k-1}} \geq - \alpha^{k-1} H \bone$.
    From this result and the fact that $w_0 = \bzero$,
    \begin{align}
        w_1 - A_1 \vf{\pi'_1}
        \leq
        \pi_1 \paren*{
            \gamma P w_0
            + \gamma H \bone
            + E_1
        }
        \leq
        \gamma H \bone + \pi_1 E_1\,.
    \end{align}
    Therefore, the inequality \eqref{eq:m s_k - vpi prime bound upper bound} holds for $k=1$.
    From the step (d) above and induction, it is straightforward to verify that the inequality \eqref{eq:m s_k - vpi prime bound upper bound} holds for other $k$.
\end{proof}

%% file: proof/E_k_bound.tex
The next lemma is necessary to bound $E_k$ by using the Azuma-Hoeffding inequality (\cref{lemma:hoeffding}).

\begin{lemma}\label{lemma:v is bounded}
    For any $k \in [K]$, $v_{k-1}$ is bounded by $H$.
\end{lemma}

\begin{proof}
    We prove the claim by induction.
    The claim holds for $k=1$ since $v_0 = \bzero$ by definition.
    Assume that $v_{k-1}$ is bounded by $H$ for some $k \geq 1$.
    Then, from the greediness of the policies $\pi_k$ and $\pi_{k-1}$,
    \begin{align}
        \pi_{k-1} q_k
        = \pi_{k-1} (s_k - \alpha s_{k-1})
        \leq v_k
        \leq \pi_k (s_k - \alpha s_{k-1})
        = \pi_k q_k
    \end{align}
    Since $q_k = r + \gamma \widehat{P}_{k-1} v_{k-1}$ is bounded by $H$
    due to the induction hypothesis, the claim holds.
\end{proof}

\begin{proof}[Proof of \cref{lemma:E_k bound}]
    Consider a fixed $k \in [K]$ and $(x, a) \in \XA$.
    Since
    \begin{align}
        E_k (x, a)
        =
        \frac{\gamma}{M} \sum_{j=1}^k \alpha^{k-j} \sum_{m=1}^{M} \underbrace{
            \paren*{ v_{j-1} (y_{j-1, m, x, a}) - P v_{j-1} (x, a) }
        }_{\text{bounded by } 2 H \text{ from \cref{lemma:v is bounded}}}\,,
    \end{align}
    $E_k (x, a)$ is a sum of bounded martingale differences
    with respect to the filtraion $(\bF_{j, m})_{j = 1, m = 1}^{k, M}$.
    Therefore, using the Azuma-Hoeffding inequality (\cref{lemma:hoeffding}),
	\begin{align}
        \P \paren*{
          \abs{E_k} (x, a)
          \geq
          3 H \sqrt{ \frac{A_\infty \iota_1}{M}}
        }
        \leq \frac{\delta}{4K \aXA}\,,
    \end{align}
    where the bound in $\P (\cdot)$ is simplified by $2 \sqrt{2} \gamma \leq 3$ and $\sum_{j=1}^k \alpha^{2 (k-j)} \leq \sum_{j=1}^k \alpha^{k-j} = A_\infty$.
    Taking the union bound over $(x, a, k) \in \XA \times [K]$,
    \begin{align}
        \P \paren*{\cE_1}
        \geq 1 - \sum_{(x, a) \in \XA} \sum_{k=1}^K \P \paren*{
          \abs{E_k} (x, a)
          \geq
          3 H \sqrt{ \frac{A_\infty \iota_1}{M}}
        }
        \geq
        1 - \frac{\delta}{4}\,,
    \end{align}
    and thus $\P \paren*{\cE_1^c} \leq \delta / 4$, which is the desired result.
\end{proof}

%% file: proof/non_stationary_coarse_policy_bound.tex
\begin{proof}[Proof of \cref{lemma:non-stationary coarse bound}]
    We condition the proof by the event $\Eone$.
    This event occurs with probability at least $1 - \delta / 4$.
    Note that under the current setting of $\alpha$, $A_\infty = H$.
    From \cref{lemma:E_k bound} and the settings of $\alpha$ and $M$,
    \begin{align}
        \sum_{j=0}^{k-1} \gamma^j \paren*{
            \pi_k P_{k-j}^{k-1} - \pi_* P_*^j
        } E_{k-j}
        \leq
        2 \sum_{j=0}^{k-1} \gamma^j \infnorm{E_{k-j}}
        \leq
        \frac{\square H \sqrt{H} \varepsilon}{\sqrt{c_2}}\,.
    \end{align}
    Thus, from \cref{lemma:non-stationary error propagation},
    $
        \vf{*} - \vf{\pi'_k}
        \leq
        \square \sqrt{H / c_2} \varepsilon + 2 (H + k) \gamma^k \bone
    $.
    Finally, using \cref{lemma:k gamma to k-th inequality},
    \begin{align}
        2 (H + K) \gamma^K
        \leq \frac{\square \varepsilon}{c_1},
    \end{align}
    and thus,
    \begin{align}
        \infnorm{\vf{*} - \vf{\pi'_K}}
        \leq
        \square \varepsilon \sqrt{\frac{H}{c_2}}
        + \frac{\square \varepsilon}{c_1}
        \leq
        \square \paren*{ \frac{1}{c_1} + \frac{1}{\sqrt{c_2}} } \sqrt{H} \varepsilon\,.
    \end{align}
    Therefore, for some $c_1$ and $c_2$, the claim holds.
\end{proof}

%% file: proof/v_error_propagation.tex
We first prove an intermediate result.

\begin{lemma}\label{lemma:pre v error prop}
    For any $k \in [K]$,
    \begin{align}
        \vf{\pi'_{k-1}}
        + \sum_{j=0}^{k-1} \gamma^j \pi_{k-1} P_{k-1-j}^{k-2} \varepsilon_{k-j}
        - \gamma^k H \bone
        \leq
        v_k
        \leq
        \vf{\pi'_k}
        + \sum_{j=0}^{k-1} \gamma^j \pi_k P_{k-j}^{k-1} \varepsilon_{k-j}
        + \gamma^k H \bone\,.
    \end{align}
\end{lemma}

\begin{proof}
    From the greediness of $\pi_{k-1}$,
    $
        v_k
        =
        w_k - \alpha w_{k-1}
        \leq
        \pi_k (s_k - \alpha s_{k-1})
        =
        \pi_k (
            r + \gamma P v_{k-1} + \varepsilon_k
        )
    $.
    By induction on $k$, therefore,
    \begin{align}
        v_k
        \leq
        \sum_{j=0}^{k-1} \gamma^j \pi_k P_{k-j}^{k-1} \paren*{
            r + \varepsilon_{k-j}
        }
        + \underbrace{\gamma^k \pi_k P_0^{k-1} v_0}_{= \bzero}
        =
        \sum_{j=0}^{k-1} \gamma^j \pi_k P_{k-j}^{k-1} \paren*{
            r + \varepsilon_{k-j}
        }\,,
    \end{align}
    Note that
    \begin{align}
        T_0^{k-1} \qf{\pi_0}
        =
        \sum_{j=0}^{k-1} \gamma^j P_{k-j}^{k-1} r + \gamma^k \underbrace{P_0^{k-1} \qf{\pi_0}}_{\geq - H\bone}
        \implies
        \sum_{j=0}^{k-1} \gamma^j P_{k-j}^{k-1} r
        \leq
        T_0^{k-1} \qf{\pi_0} + \gamma^k H\,.
    \end{align}
    Accordingly, 
    $
        v_k
        \leq
        \pi_k T_0^{k-1} \qf{\pi_0}
        + \sum_{j=0}^{k-1} \gamma^j \pi_k P_{k-j}^{k-1} \varepsilon_{k-j}
        + \gamma^k H \bone\,.
    $
    
    Similarly, from the greediness of $\pi_k$,
    $
        v_k
        =
        w_k - \alpha w_{k-1}
        \geq
        \pi_{k-1} (s_k - \alpha s_{k-1})
        \geq
        \pi_{k-1} (
            r + \gamma P v_{k-1} + \varepsilon_k
        )
    $.
    By induction on $k$, therefore,
    \begin{align}
        v_k
        \geq
        \sum_{j=0}^{k-1} \gamma^j \pi_{k-1} P_{k-1-j}^{k-2} \paren*{
            r + \varepsilon_{k-j}
        } 
        + \underbrace{\gamma^{k-1} \pi_{k-1} P_0^{k-2} P v_0}_{=\bzero}\,.
    \end{align}
    Note that $T_0^{k-2} \qf{\pi_0} = T_0^{k-2} (r + \gamma P \vf{\pi_0})$, and
    \begin{align}
        T_{0}^{k-2} \qf{\pi_0}
        =
        \sum_{j=0}^{k-1} \gamma^j P_{k-1-j}^{k-2} r + \gamma^k \underbrace{P_0^{k-2} P \vf{\pi_0}}_{\leq H \bone}
        \implies
        \sum_{j=0}^{k-1} \gamma^j P_{k-1-j}^{k-2} r
        \geq
        T_0^{k-2} \qf{\pi_0} - \gamma^k H\,.
    \end{align}
    Accordingly, 
    $
        v_k
        \geq
        \pi_{k-1} T_0^{k-2} \qf{\pi_0}
        + \sum_{j=0}^{k-1} \gamma^j \pi_{k-1} P_{k-1-j}^{k-2} \varepsilon_{k-j}
        - \gamma^k H \bone\,.
    $
\end{proof}

\begin{proof}[Proof of \cref{lemma:v error prop}]
    From \cref{lemma:pre v error prop} and $\pi_k T^{\pi_{k-1:1}} \qf{\pi_0} = \vf{\pi'_k} \leq \vf{*}$, we have that
    \begin{align}
        \vf{\pi'_{k-1}}
        + \sum_{j=0}^{k-1} \gamma^j \pi_{k-1} P_{k-1-j}^{k-2} \varepsilon_{k-j}
        - 2 \gamma^k H \bone
        \leq
        v_k
        \leq
        \vf{*}
        + \sum_{j=0}^{k-1} \gamma^j \pi_k P_{k-j}^{k-1} \varepsilon_{k-j}
        + 2 \gamma^k H \bone\,,
    \end{align}
    where we loosened the bound by multiplying $\gamma^k H$ by $2$.
    By simple algebra, the lower bound for $\vf{*} - v_k$ is obtained.
    On the other hand,
    from \cref{lemma:non-stationary error propagation},
    \begin{align}
        \vf{\pi'_{k-1}}
        \geq
        \vf{*}
        - \frac{1}{A_\infty} \sum_{j=0}^{k-2} \gamma^j \paren*{
            \pi_{k-1} P_{k-1-j}^{k-2} - \pi_* P_*^j
        } E_{k-1-j}
        - 2 H \paren*{ \alpha^{k-1} + \frac{A_{\gamma, k-1}}{A_\infty} } \bone
    \end{align}
    for any $k \in \{2, \ldots, K\}$.
    Therefore, we have that
    \begin{align}
        \vf{*} - v_k
        &\leq
        2 H \paren*{ \alpha^{k-1} + \gamma^k + \frac{A_{\gamma, k-1}}{A_\infty} } \bone
        \\
        &\hspace{2em}
        + \frac{1}{A_\infty} \sum_{j=0}^{k-2} \gamma^j \paren*{
            \pi_{k-1} P_{k-1-j}^{k-2} - \pi_* P_*^j
        } E_{k-1-j}
        - \sum_{j=0}^{k-1} \gamma^j \pi_{k-1} P_{k-1-j}^{k-2} \varepsilon_{k-j}
    \end{align}
    for any $k \in \{2, \ldots, K\}$.
    
    Finally, for $k=1$, since $v_1 = \pi_1 q_1 = \pi_1 r$,
    \begin{align}
        - \gamma H \bone
        \leq
        \pi_* \paren*{\qf{*} - r}
        \leq
        \vf{*} - v_1
        \leq
        \gamma \pi_* P \vf{*}
        \leq \gamma H \bone\,.
    \end{align}
    As $\Gamma_1 \geq \bzero$, the claim holds for $k=1$ too.
\end{proof}

%% file: proof/coarse_v_bound.tex
\begin{proof}[Proof of \cref{lemma:eps_k bound}]
    Consider a fixed $k \in [K]$ and $(x, a) \in \XA$.
    Since
    \begin{align}
        \varepsilon_k (x, a) = \frac{\gamma}{M} \sum_{m=1}^{M} \underbrace{
            \paren*{ v_{k-1} (y_{k-1, m, x, a}) - P v_{k-1} (x, a) }
        }_{\text{bounded by } 2H \text{ from \cref{lemma:v is bounded}}}\,,
    \end{align}
    $\varepsilon_k (x, a)$ is a sum of martingale differences with respect to the filtraion $(\bF_{k, m})_{m = 1}^{M}$ and bounded by $2 \gamma H / M$.
    Therefore, using the Azuma-Hoeffding inequality (\cref{lemma:hoeffding}),
	\begin{align}
        \P \paren*{
          \abs{\varepsilon_k} (x, a)
          \geq
          3 H \sqrt{\frac{\iota_1}{M}}
        }
        \leq \frac{\delta}{4 K \aXA}\,,
    \end{align}
    where the bound in $\P (\cdot)$ is simplified by $2 \sqrt{2} \leq 3$.
    Taking the union bound over $(x, a, k) \in \XA \times [K]$,
    \begin{align}
        \P \paren*{\cE_2}
        \geq 1 - \sum_{(x, a) \in \XA} \sum_{k=1}^K \P \paren*{
          \abs{\varepsilon_k} (x, a)
          \geq
          3 H \sqrt{\frac{\iota_1}{M}}
        }
        \geq
        1 - \frac{\delta}{4}\,,
    \end{align}
    and thus $\P \paren*{\cE_2^c} \leq \delta / 4$, which is the desired result.
\end{proof}

Next, we prove a uniform bound on $\vf{*} - v_k$.

\begin{lemma}\label{lemma:coarse v bound}
    Conditioned on $\Eone \cap \Etwo$,
    \begin{align}
        \norm{\vf{*} - v_k}_\infty
        <
        2 H \min \brace*{
            1,
            \gamma^k
            + \alpha^{k-1}
            + \frac{A_{\gamma, k-1}}{A_\infty}
            + 6 H \sqrt{\frac{\iota_1}{M}}
        }
    \end{align}
    for all $k \in [K]$, where $1/0 \df \infty$.
\end{lemma}

\begin{proof}
    Let
    $
        e_k \df \displaystyle \gamma^k H + H \max_{j \in [k]} \infnorm{\varepsilon_j}
    $.
    From \cref{lemma:v error prop},
    $
        \vf{*} - v_k
        \geq
        - 2 e_k \bone
    $
    for any $k \in [K]$, and
    \begin{align}
        \vf{*} - v_k
        &\leq
        2 H \paren*{
            \alpha^{k-1}
            + \frac{A_{\gamma, k-1}}{A_\infty}
            + \frac{1}{A_\infty} \max_{j \in [k-1]} \infnorm{E_j}
        } \bone + 2 e_k \bone
    \end{align}
    for any $k \in \{ 2, \ldots, K \}$.
    Note that $\infnorm{\vf{*} - v_k} \leq 2H$ from \cref{lemma:v is bounded} for any $k$.
    Combining these results with \cref{lemma:E_k bound,lemma:eps_k bound},
    \begin{align}
        \norm{\vf{*} - v_k}_\infty
        &<
        2 H \min \brace*{
            1,
            \gamma^k
            + \alpha^{k-1}
            + \frac{A_{\gamma, k-1}}{A_\infty}
            + 3 H \sqrt{\frac{\iota_1}{M}} \paren*{1 + \sqrt{\frac{1}{A_\infty}}}
        }
        \\
        &\leq
        2 H \min \brace*{
            1,
            \gamma^k
            + \alpha^{k-1}
            + \frac{A_{\gamma, k-1}}{A_\infty}
            + 6 H \sqrt{\frac{\iota_1}{M}}
        }
    \end{align}
    for all $k \in [K]$, where we used the fact that $1 \leq A_\infty$.
    This concludes the proof.
\end{proof}

Now, we are ready to prove \cref{lemma:variance upper bounds}.

\begin{proof}[Proof of \cref{lemma:variance upper bounds}]
    Clearly $\sigma (v_0) = \bzero$ since $v_0 = \bzero$.
    From \cref{lemma:variance decomposition},
    $
        \sigma(v_k)
        \leq
        \sigma \paren*{v_k - \vf{*}} + \sigma(\vf{*})\,.
    $
    Using Popoviciu's inequality on variances (\cref{lemma:popoviciu}) together with \cref{lemma:coarse v bound},
    \begin{align}
        \sigma \paren*{v_k - \vf{*}}
        &\leq
        2 H \min \brace*{
            1,
            \gamma^k
            + \alpha^{k-1}
            + \frac{A_{\gamma, k-1}}{A_\infty}
            + 6 H \sqrt{\frac{\iota_1}{M}}
        }\,,
    \end{align}
    where we used a simple formula, $\min \{ a, b \}^2 = \min \{ a^2, b^2 \}$ for any scalars $a, b \geq 0$.
    Finally, loosening the bound by replacing $\gamma^k + \alpha^{k-1}$ by $2 \max\{\alpha, \gamma\}^{k-1}$,
    the claim holds.
\end{proof}

%% file: proof/refined_eps_k_and_E_k_bound.tex
\begin{proof}[Proof of \cref{lemma:refined E_k bound}]
    Consider a fixed $k \in [K]$ and $(x, a) \in \XA$.
    Since
    \begin{align}
        E_k (x, a)
        =
        \frac{\gamma}{M} \sum_{j=1}^k \alpha^{k-j} \sum_{m=1}^{M} \underbrace{
            \paren*{ v_{j-1} (y_{j-1, m, x, a}) - P v_{j-1} (x, a) }
        }_{\text{bounded by } 2 H \text{ from \cref{lemma:v is bounded}}}\,,
    \end{align}
    $E_k (x, a)$ is a sum of bounded martingale differences
    with respect to the filtraion $(\bF_{j, m})_{j = 1, m = 1}^{k, M}$.
    From the facts that $v_0=\bzero$, and $\gamma \leq 1$,
    \begin{align}
        V'
        = \frac{\gamma^2}{M} \sum_{j=1}^k \alpha^{2 (k-j)} \PVar \paren*{v_{j-1}} (x, a)
        \leq
        \frac{1}{M} \underbrace{
            \sum_{j=2}^k \alpha^{2 (k-j)} \PVar \paren*{v_{j-1}} (x, a)
        }_{\df \heartsuit}\,,
    \end{align}
    Since we are conditioned with the event $\Eone \cap \Etwo$, the inequality \eqref{eq:variance upper bounds} in \cref{lemma:variance upper bounds} holds and implies that the predictable quadratic variation $V'$ satisfies the following inequality:
    \begin{align}
        \heartsuit
        &\leq
        \sum_{j=2}^k \alpha^{2 (k-j)} \paren*{
            \sigma (\vf{*}) (x, a)
            + 2H
            \min \brace*{
                1,
                2 \max \brace{\alpha, \gamma}^{j-2}
                + \frac{A_{\gamma, j-2}}{A_\infty}
                + 6 H \sqrt{\frac{\iota_1}{M}}
            }
        }^2
        \\
        &\leq
        \sum_{j=2}^k \alpha^{2 (k-j)} \paren*{
            \sigma (\vf{*}) (x, a)
            + 2H
            \paren*{
                2 \max \brace{\alpha, \gamma}^{j-2}
                + \frac{A_{\gamma, j-2}}{A_\infty}
                + 6 H \sqrt{\frac{\iota_1}{M}}
            }
        }^2
        \\
        &\leq
        4 \sum_{j=2}^k \alpha^{2 (k-j)} \paren*{
            \PVar (\vf{*}) (x, a)
            + 4 H^2
            \paren*{
                4 \max \brace{\alpha, \gamma}^{2(j-2)}
                + \frac{A_{\gamma, j-2}^2}{A_\infty^2}
                + \frac{36 H^2 \iota_1}{M}
            }
        }\,,
    \end{align}
    where the last line follows from \cref{lemma:square inequality}.
    Consequently, $V'$ is bounded by
    \begin{align}
        V'
        &\leq
        \frac{4}{M} \sum_{j=2}^k \alpha^{2 (k-j)} \paren*{
            \PVar (\vf{*}) (x, a)
            + 4 H^2
            \paren*{
                4 \max \brace{\alpha, \gamma}^{2(j-2)}
                + \frac{A_{\gamma, j-2}^2}{A_\infty^2}
                + \frac{36 H^2 \iota_1}{M}
            }
        }\,,
    \end{align}
    which is equal to $V_k (x, a)$.
    Using \cref{lemma:conditional bernstein} and taking the union bound over $(x, a, k) \in \XA \times [K]$,
	\begin{align}
        \P \parenc*{
          \exists (x, a, k) \in \XA \times [K]
          \text{ s.t. }
          \abs{E_K} (x, a)
          \geq
          \frac{4 H \iota_2}{3 M} + \sqrt{2 V_k (x, a) \iota_2}
        }{
            \cE_1 \cap \cE_2
        }
        \leq \frac{\delta}{4}\,.
    \end{align}
    (Recall that $\P (\Eone \cap \Etwo) \geq 1 - \dfrac{\delta}{2} \geq \dfrac{1}{2}$,
    and hence, we need to use $\iota_2$.)
    Thus, $\P \parenc*{\cE_3^c}{\cE_1 \cap \cE_2} \leq \dfrac{\delta}{4}$.
\end{proof}

\begin{proof}[Proof of \cref{lemma:refined eps_k bound}]
    Consider a fixed $k \in [K]$ and $(x, a) \in \XA$.
    Since
    \begin{align}
        \varepsilon_k (x, a)
        =
        \frac{\gamma}{M} \sum_{m=1}^{M} \underbrace{
            \paren*{ v_{k-1} (y_{k-1, m, x, a}) - P v_{k-1} (x, a) }
        }_{\text{bounded by } 2 H \text{ from \cref{lemma:v is bounded}}}\,,
    \end{align}
    $\varepsilon_k (x, a)$ is a sum of bounded martingale differences
    with respect to $\bF_{k, m}$.
    Since we are conditioned with the event $\Eone \cap \Etwo$, the inequality \eqref{eq:variance upper bounds} in \cref{lemma:variance upper bounds} holds and implies that the predictable quadratic variation $V'$ can be shown to 
    satisfy the following inequality as in the proof of \cref{lemma:refined E_k bound}:
    \begin{align}
        V'
        = \frac{\gamma^2}{M} \PVar \paren*{v_{k-1}} (x, a)
        \leq
        \frac{4}{M}
        \overline{\PVar}_k
        \,,
    \end{align}
    where the last line is equal to $W_k (x, a)$. (Note that $v_0 = \bzero$.)
    
    Using \cref{lemma:conditional bernstein} and taking the union bound over $(x, a, k) \in \XA \times [K]$,
	\begin{align}
        \P \parenc*{
          \exists (x, a, k) \in \XA \times [K]
          \text{ s.t. }
          \abs{\varepsilon_k} (x, a)
          \geq
          \frac{4 H \iota_2}{3 M} + \sqrt{2 W_k (x, a) \iota_2}
        }{
            \cE_1 \cap \cE_2
        }
        \leq \frac{\delta}{4}\,.
    \end{align}
    (Recall that $\P (\Eone \cap \Etwo) \geq 1 - \dfrac{\delta}{2} \geq \dfrac{1}{2}$,
    and hence, we need to use $\iota_2$.)
    Thus, $\P \parenc*{\cE_4^c}{\cE_1 \cap \cE_2} \leq \dfrac{\delta}{4}$.
\end{proof}

%% file: proof/last_policy_error_propagation.tex
To prove \cref{lemma:error propagation}, we need the following lemma.

\begin{lemma}\label{lemma:delta_k bound}
    For any $k \in [K]$, let $\Delta_k \df w_k - w_{k-1}$.
    Then, for any $k \in [K]$,
    \begin{align}
        \pi_{k-1} \sum_{j=0}^{k-1} \gamma^j P_{k-1-j}^{k-2} E_{k-j}'
        - A_{\gamma, k} \bone
        \leq
        \Delta_k
        \leq
        \pi_k \sum_{j=0}^{k-1} \gamma^j P_{k-j}^{k-1} E_{k-j}'
        + A_{\gamma, k} \bone\,.
    \end{align}
\end{lemma}

\begin{proof}
  We prove only the upper bound by induction as the proof for a lower bound is similar.
  We have that
  $
        \Delta_k
        =
        \pi_k s_k - \pi_{k-1} s_{k-1}
        \leq
        \pi_k \paren{ s_k - s_{k-1} }
  $,
  where the inequality follows from the greediness of $\pi_{k-1}$.
  Let $\heartsuit_k \df s_k - s_{k-1}$.
  Since $s_0 = \bzero$,
  $
      \heartsuit_1
      = r + E_1'
      \leq \bone + E_1'
  $.
  From the monotonicity of $\pi_1$, the claim holds for $k=1$.
  Assume that for some $k-1 \geq 1$, the claim holds.
  Then, from the equation \eqref{eq:rewrite s_k}, the induction hypothesis,
  and the monotonicity of $P$,
  \begin{align}
    \heartsuit_k
    &=
    \paren*{A_k - A_{k-1}} r
    + \gamma P \Delta_{k-1}
    + E_k'
    \\
    &\leq
    \sum_{j=0}^{k-1} \gamma^j P_{k-j}^{k-1} E_{k-j}'
    + (\alpha^{k-1} + \gamma A_{\gamma, k-1}) \bone
    =
    \sum_{j=0}^{k-1} \gamma^j P_{k-j}^{k-1} E_{k-j}'
    + A_{\gamma, k} \bone\,.
  \end{align}
  The claimed upper bound follows from the monotonicity of $\pi_k$.
\end{proof}

Now, we are ready to prove \cref{lemma:error propagation}.

\begin{proof}[Proof of \cref{lemma:error propagation}]
    Note that
    \begin{align}
        \bzero
        \leq
        \vf{*} - \vf{\pi_k}
        =
        \frac{A_k}{A_\infty} \paren*{ \vf{*} - \vf{\pi_k} }
        + \alpha^k \paren*{ \vf{*} - \vf{\pi_k} }
        \leq
        \frac{A_k}{A_\infty} \paren*{ \vf{*} - \vf{\pi_k} }
        + 2 H \alpha^k \bone
    \end{align}
    since $\vf{*} - \vf{\pi_k} \leq 2 H \bone$.
    Therefore, we need an upper bound for $A_k (\vf{*} - \vf{\pi_k})$.
    We decompose $A_k (\vf{*} - \vf{\pi_k})$ to $A_k \vf{*} - w_k$ and $w_k - A_k\vf{\pi_k}$.
    Then, we derive upper bounds for each of them.
    The desired result is obtained by summing up those bounds.
    
    \paragraph{Upper bound for $A_k \vf{*} - w_k$.}
    Note that
    \begin{align}
        A_k \vf{*} - w_k
        &\numeq{=}{a}
        \cN^{\pi_*} \paren*{
            \pi_* \paren*{ A_k r + \gamma P w_k} - w_k
        }
        \\
        &\numeq{\leq}{b}
        \cN^{\pi_*} \pi_* \paren*{
            A_k r + \gamma P w_k - s_k
        }
        \\
        &\numeq{=}{c}
        \cN^{\pi_*} \pi_* \paren*{
            \gamma P \paren*{w_k - w_{k-1}} - E_k
        }
        \\
        &\numeq{\leq}{d}
        \cN^{\pi_*} \pi_* \paren*{
            \sum_{j=1}^k \gamma^j P_{k+1-j}^{k} E_{k+1-j}' - E_k
        }
        + H A_{\gamma, k} \bone\,,
    \end{align}
    where (a) is due to the fact that
    $I = \cN^\pi (I - \gamma \pi P)$
    and
    $\vf{\pi} = \cN^{\pi} \pi r$
    for any policy $\pi$,
    (b) is due to the greediness of $\pi_k$,
    (c) follows from the equation \eqref{eq:rewrite s_k},
    and (d) follows from \cref{lemma:delta_k bound}.

    \paragraph{Upper bound for $w_k - A_k \vf{\pi_k}$.}
    We have that
    \begin{align}
        w_k - A_k \vf{\pi_k}
        &\numeq{=}{a}
        \cN^{\pi_k} \paren*{
            w_k - \pi_k \paren*{ A_k r + \gamma P w_k}
        }
        \\
        &\numeq{=}{b}
        \cN^{\pi_k} \pi_k \paren*{
            w_k - A_k r - \gamma P w_k
        }
        \\
        &\numeq{=}{c}
        \cN^{\pi_k} \pi_k \paren*{
            - \gamma P \paren*{w_k - w_{k-1}} + E_k
        }
        \\
        &\numeq{\leq}{d}
        \cN^{\pi_k} \pi_k \paren*{
            E_k - \sum_{j=1}^k \gamma^j P_{k-j}^{k-1} E_{k+1-j}'
        }
        + H A_{\gamma, k} \bone\,,
    \end{align}
    where (a) is due to the fact that
    $I = \cN^\pi (I - \gamma \pi P)$
    and
    $\vf{\pi} = \cN^{\pi} \pi r$
    for any policy $\pi$,
    (b) is due to the definition of $w_k$,
    (c) follows from the equation \eqref{eq:rewrite s_k},
    and (d) follows from \cref{lemma:delta_k bound}.
\end{proof}

%% file: proof/last_policy_coarse_policy_bound.tex
Before starting the proof, we note that $A_\infty = H^2$ under the current setting.

\begin{proof}[Proof of \cref{lemma:coarse bound for last policy}]
    From \cref{lemma:E_k bound},
    $
        \infnorm{E_k}
        \leq
        3 H \sqrt{A_\infty \iota_1 / M}
        \leq
        3 \varepsilon \sqrt{H^3 / c_4}
    $ for any $k \in [K]$.
    On the other hand, from \cref{lemma:eps_k bound},
    $
        \infnorm{\varepsilon_k}
        \leq
        3 H \sqrt{\iota_1 / M}
        \leq
        3 \varepsilon \sqrt{H / c_4}
    $ for any $k \in [K]$.
    Combining these bounds
    with \cref{lemma:non-stationary error propagation},
    \begin{align}
        \vf{*} - \vf{\pi'_k}
        \leq
        \frac{\square}{H} \max_{j \in [k]} \infnorm{E_j} \bone
        + \square H \alpha^k \bone
        \leq
        \square \paren*{
            \varepsilon \sqrt{\frac{H}{c_4}} + H \alpha^k
        } \bone
    \end{align}
    for any $k \in [K]$,
    where we used the fact that $A_{\gamma, k} / A_{\infty} \leq \alpha^k / H \leq \alpha^k$, which follows from \cref{lemma:A_gamma_k bound}, is used.
    Furthermore, combining previous upper bounds for errors with \cref{lemma:error propagation},
    \begin{align}
        \vf{*} - \vf{\pi_k}
        &\leq
        2 H \underbrace{
            \paren*{\alpha^k + \frac{A_{\gamma, k}}{A_\infty}}
        }_{\leq 2 \alpha^k \text{ from (a)}} \bone
        + \frac{1}{A_\infty} \underbrace{
            \paren*{\cN^{\pi_k} \pi_k - \cN^{\pi_*} \pi_*} E_k
        }_{
            \leq 2H \infnorm{E_k} \bone \text{ from (b)}
        }
        \\
        &\hspace{6em}+ \frac{1}{A_\infty} \sum_{j=1}^k \gamma^j \underbrace{
            \paren*{
                \cN^{\pi_*} \pi_* P_{k+1-j}^k - \cN^{\pi_k} \pi_k P_{k-j}^{k-1}
            } E_{k+1-j}'
        }_{\leq 2 H ( \infnorm{\varepsilon_{k+1-j}} + (1-\alpha) \infnorm{E_{k-j}} ) \bone \text{ from (c)} }
        \\
        &\numeq{\leq}{d}
        4 H \alpha^k \bone
        + \frac{2}{H} \infnorm{E_k}
        + 2 \max_{j \in [k]} \paren*{
            \infnorm{\varepsilon_j} + \frac{1}{H^2} \infnorm{E_j}
        }
        \\
        &\leq
        4 H \alpha^k \bone
        + 6 \varepsilon \sqrt{\frac{H}{c_4}}
        + \frac{6 \varepsilon}{\sqrt{c_4}} \paren*{
            \sqrt{H} + \frac{1}{\sqrt{H}}
        } \bone
        = \square \paren*{
            \varepsilon \sqrt{\frac{H}{c_4}} + H \alpha^k
        } \bone
    \end{align}
    for any $k \in [K]$,
    where (a) follows as $A_{\gamma, k} / A_{\infty} \leq \alpha^k / H \leq \alpha^k$ from \cref{lemma:A_gamma_k bound},
    (b) is due to the monotonicity of stochastic matrices,
    and $- \infnorm{E_k} \bone \leq E_k \leq \infnorm{E_k} \bone$ for any $k \in [K]$,
    (c) is due to the monotonicity of stochastic matrices,
    and
    $
    - (\infnorm{\varepsilon_k} + (1-\alpha) \infnorm{E_{k-1}} ) \bone
    \leq
    E_k'
    \leq
    (\infnorm{\varepsilon_k} + (1-\alpha) \infnorm{E_{k-1}} ) \bone
    $ for any $k \in [K]$,
    and (d) follows by taking the maximum over $j$.
\end{proof}

%% file: appendix/illustration_details.tex
\section{Details on empirical illustrations}
\label{app:illustrations}

This appendix details the settings used for the illustrations of \cref{sec:illustrations}. It provides

\begin{itemize}
    \item a precise definition of the Garnet setting and pseudo-code for \qlearning in \cref{subappx:detailed};
    \item additional numerical experiments illustrating the effects of $\alpha$ and $M$ on the algorithm in \cref{subappx:additional}.
\end{itemize}

\subsection{Detailed setting}
\label{subappx:detailed}

\paragraph{Garnets.} We use the Garnets~\citep{archibald1995generation} class of random MDPs. A Garnet is characterized by three integer parameters, $X$, $A$, and $B$, that are respectively the number of states, the number of actions, and the branching factor -- the maximum number of accessible new states in each state. For each $(x,a)\in \XA$, we draw $B$ states ($y_1, \hdots, y_B$) from $\X$ uniformly without replacement. Then, we draw $B-1$ numbers uniformly in $(0,1)$, denoting them sorted as $(p_1, \hdots, p_{B-1})$. We set the transition probability $P^{y_k}_{x,a} = p_k - p_{k-1}$ for each $1 \leq k \leq B$, with $p_0=0$ and $p_B = 1$. Finally, the reward function, depending only on the states, is drawn uniformly in $(-1, 1)$ for each state. In our examples, we used $X=8$, $A=2$, and $B=2$. We compute our experiments with $\gamma=0.9$.

\paragraph{Q-learning.} For illustrative purposes, we compare the performance of \mdvi to the one of a sampled version of \qlearning, that we know is not minimax-optimal. For completeness, the pseudo-code for this method is given in \cref{algo:qlearning}. It shares the time complexity of \mdvi, but has a lower memory complexity, since it does not need to store an additional $XA$ table. 

\begin{algorithm}[]
    \KwIn{ number of iterations $K$, number of samples per iteration $M$, $w\in[0.5, 1]$ a learning rate parameter.}
    Let $q_0 = \bzero \in \R^{\aXA}$;\\
    \For{$k$ \textbf{\emph{from}} $0$ \textbf{\emph{to}} $K-1$}{
        \For{\textbf{\emph{each state-action pair}} $\paren*{x, a} \in \XA$}{
            Sample $(y_{k, m, x, a})_{m=1}^{M}$ from the generative model $P(\cdot|x, a)$\;
            Let $m_{k+1} (x, a) = r (x, a) + \gamma M^{-1} \sum_{m=1}^M \max_{a'} q_k (y_{k, m, x, a}, a')$\;
        }
        Let $\eta_k = (k + 1)^{-w}$; \\
        Let $q_{k+1} = (1 - \eta_k)q_{k} + \eta_k m_{k+1}$;
    }
    \Return{$\pi_K$ , a greedy policy with respect to $q_K$\;}
    \caption{$\qlearning (K, M, w)$}\label{algo:qlearning}
\end{algorithm}

\subsection{Additional numerical illustrations}
\label{subappx:additional}

\looseness=-1
\paragraph{Additional experiment for sample complexity.} In \cref{fig:sample_complexity}, we plot the sample complexity of a standard version of \qlearning using $w=1$ (\textit{i.e} performing an exact average of $q$-values). However, we know~\citep{even2003learning} that
we can reach a better sample complexity by choosing a more appropriate $w$ in $(0.5, 1)$. In \cref{fig:sample_complexity_appx}, we provide the sample complexity for \mdvi, and \qlearning with $w=1$ and $w=0.7$. The version with $w=0.7$ catches up with \mdvi at high errors, but the difference is still quite large at higher precision. Note that we add additional data points for $\varepsilon < 10^{-3}$. Both versions of \qlearning do not have sample complexity plotted for these errors, because they did not reach these $\varepsilon$ in the number of iterations we ran them (up to $10^{5}$ iterations).

\begin{figure}
    \centering
    \includegraphics[width=0.8\linewidth]{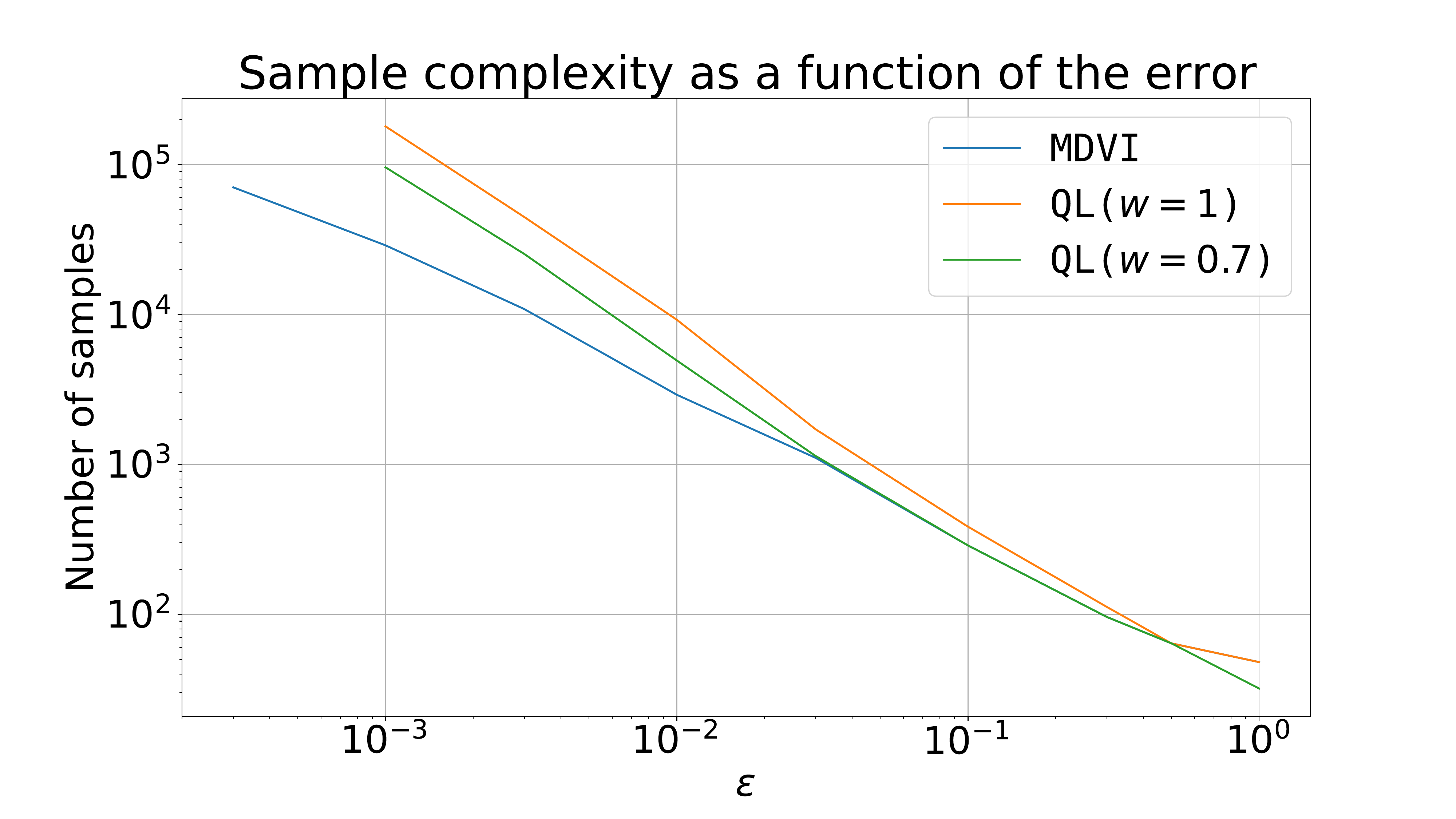}
    \caption{Number of samples needed to reach a certain error.}
    \label{fig:sample_complexity_appx}
\end{figure}

\looseness=-1
\paragraph{Influence of $\alpha$.} We showcase the impact of $\alpha$ when $M=1$ in \cref{fig:convergence_alpha}. With $\alpha=1$, \mdvi will asymptotically converge to $\pi_*$. With a $\alpha < 1$, \mdvi will reach an $\varepsilon$-optimal policy, but will not actually converge to the optimal policy of the MDP (although this $\varepsilon$ can be controlled by choosing a large enough value for $\alpha$, or a larger value of $M$). Indeed, in the latter case, the distance to the optimal policy depends on a moving average of the errors (by a factor $\alpha$). The moving average reduces the variance, but does not bring it zero, contrarily to the exact average implicitly performed when $\alpha=1$. This behaviour is illustrated in \cref{fig:convergence_alpha}. We observe there that, with $M=1$, one has to choose a large enough value of $\alpha$ to reach a policy close enough to the optimal one.

\begin{figure}
    \centering
    \includegraphics[width=0.49\linewidth]{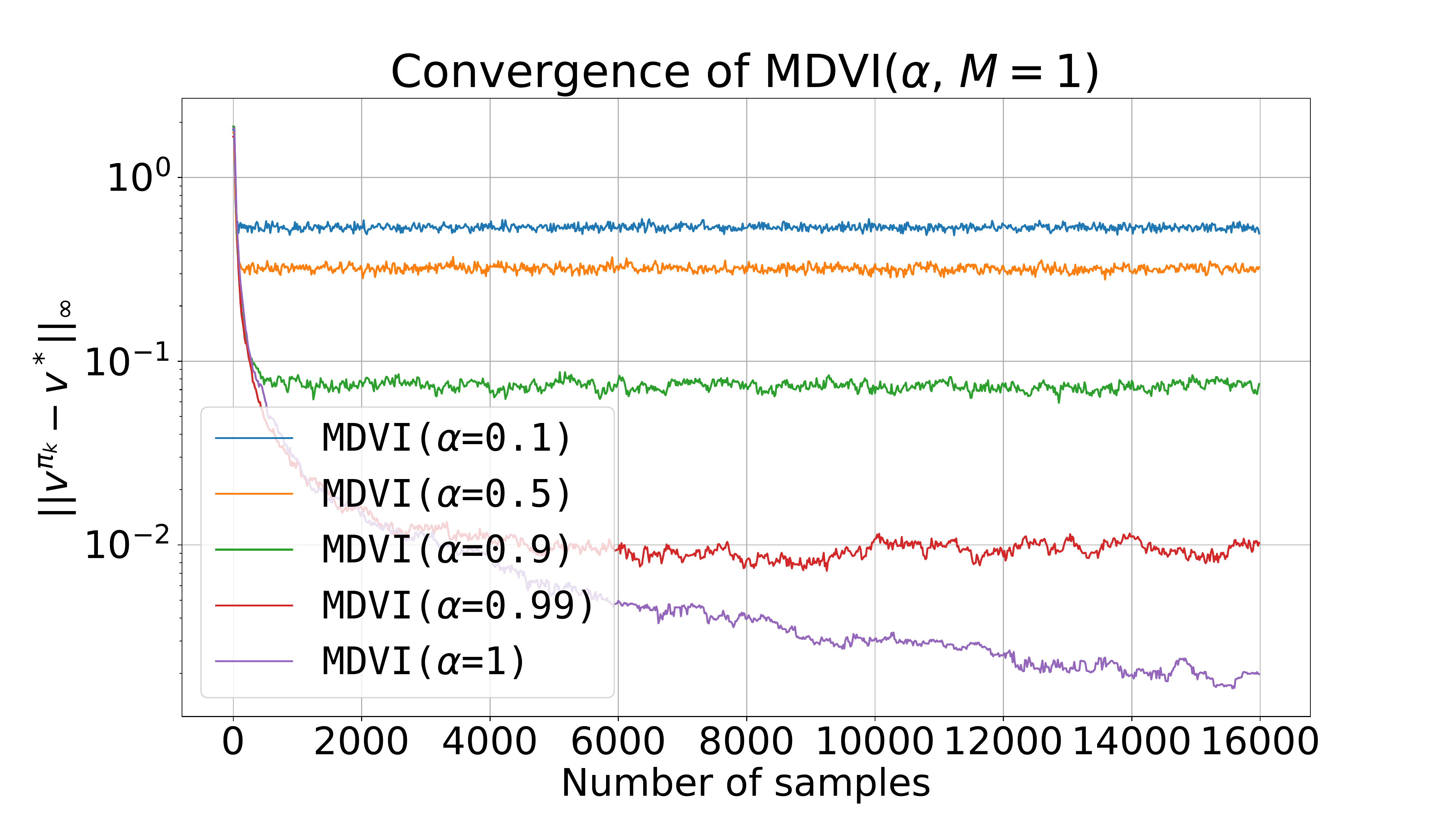}
    \includegraphics[width=0.49\linewidth]{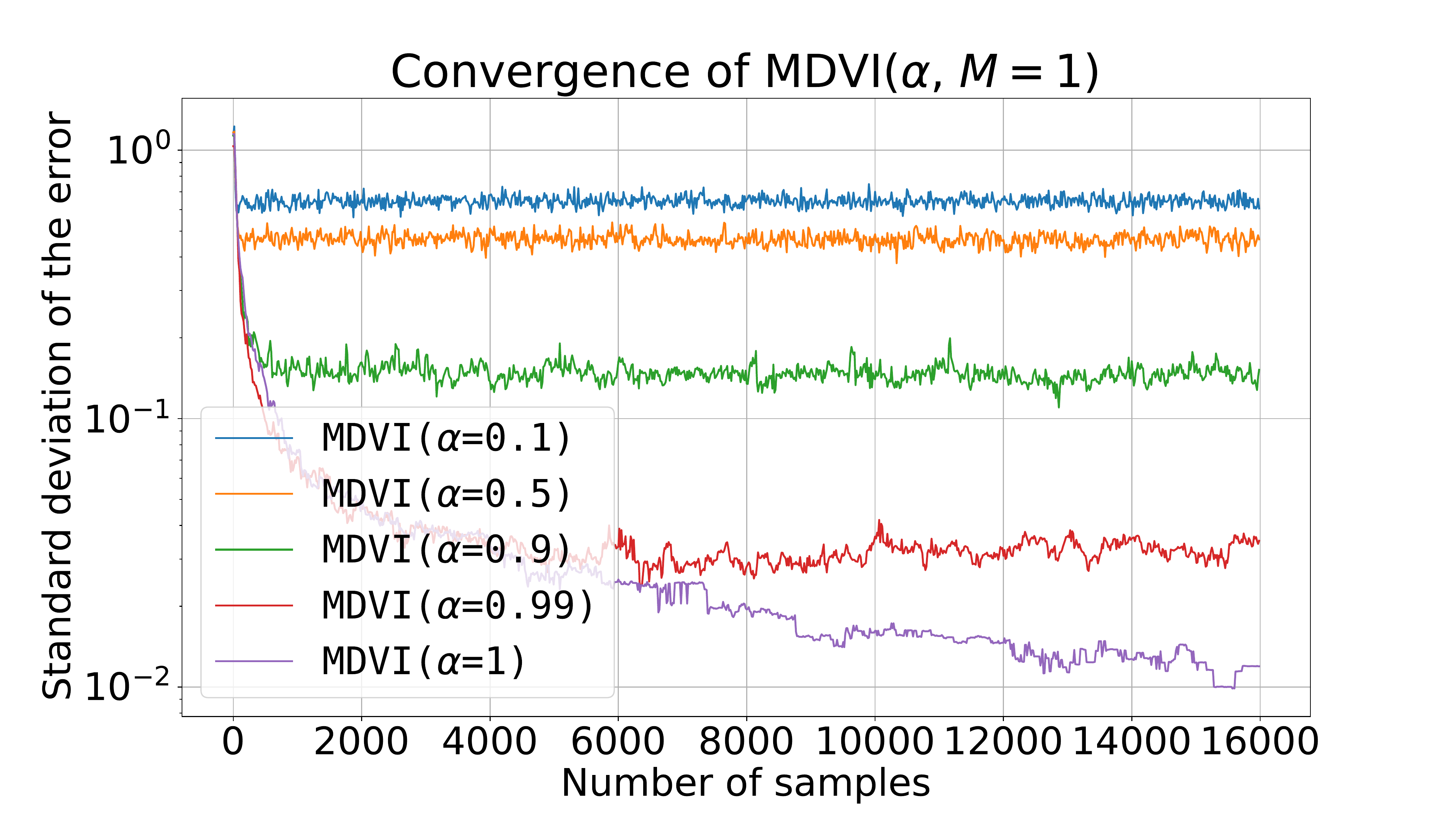}
    \caption{Error of the policy computed by \mdvi in function of the number of samples used. \textbf{Left:} mean, \textbf{Right:} standard deviation; estimated over $1000$ MDPs.}
    \label{fig:convergence_alpha}
\end{figure}

\looseness=-1
\paragraph{Influence of $M$.} Choosing the right $M$ is not that obvious from the theory (it notably depends on an unknown constant $c_2$). We illustrate in \cref{fig:convergence_m} the influence it has on the speed of convergence of \mdvi. We run \mdvi with $\alpha=0.99$ (for a setting where $\gamma=0.9$), and for different values of $M$. With a fixed $\alpha$, a larger $M$ allows \mdvi to reach a lower asymptotic error, but slows down the learning in early iterations. $M$ cannot however be chosen as large as possible: at one point it start to be useless to increase its value. For instance, moving from $M=5$ to $M=10$ does not allow for a noticeable lower error, but slows the learning. We compare this to the setting where $\alpha=1$ for
completeness.

\begin{figure}[t]
    \centering
    \includegraphics[width=0.49\linewidth]{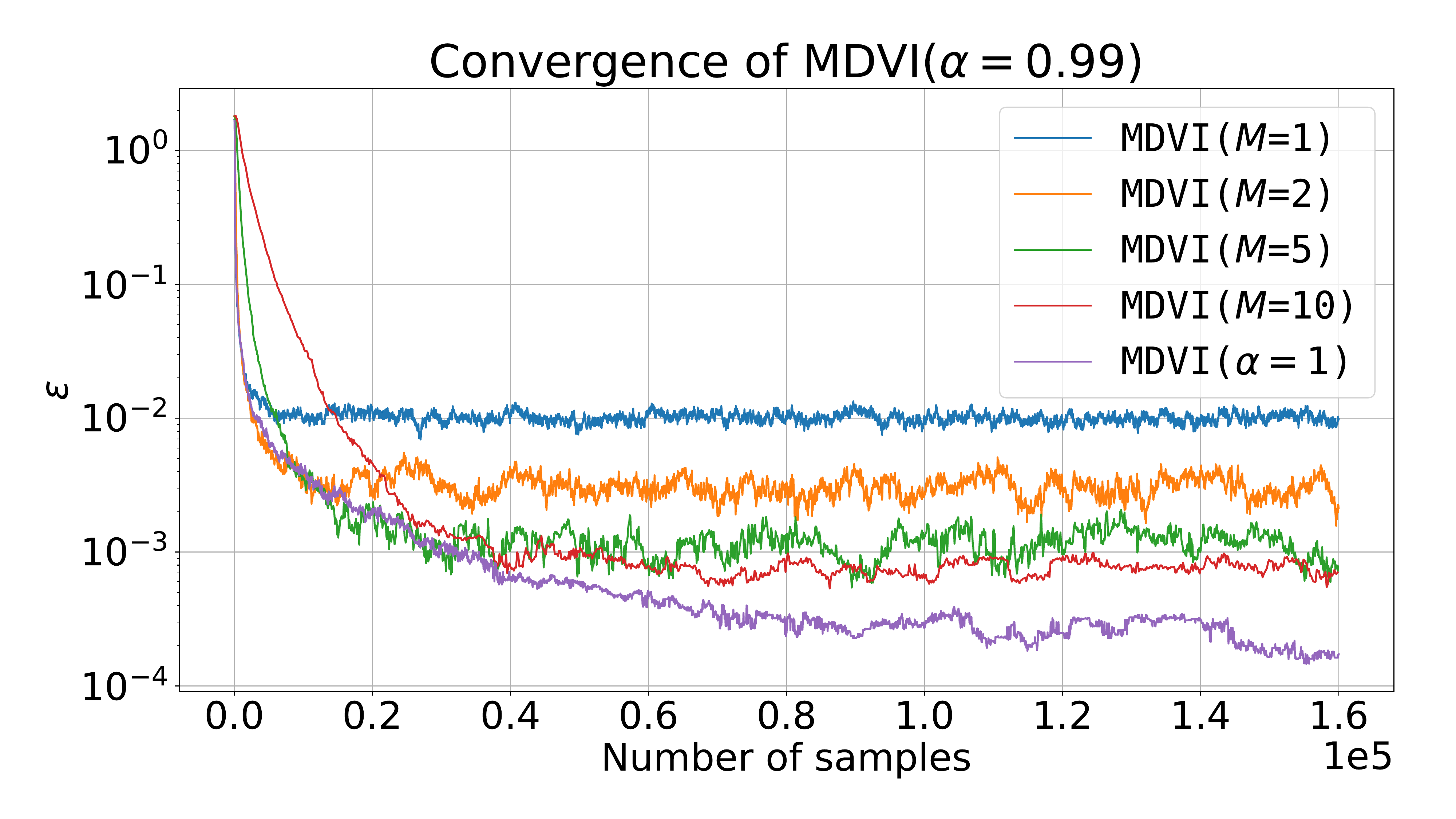}
    \includegraphics[width=0.49\linewidth]{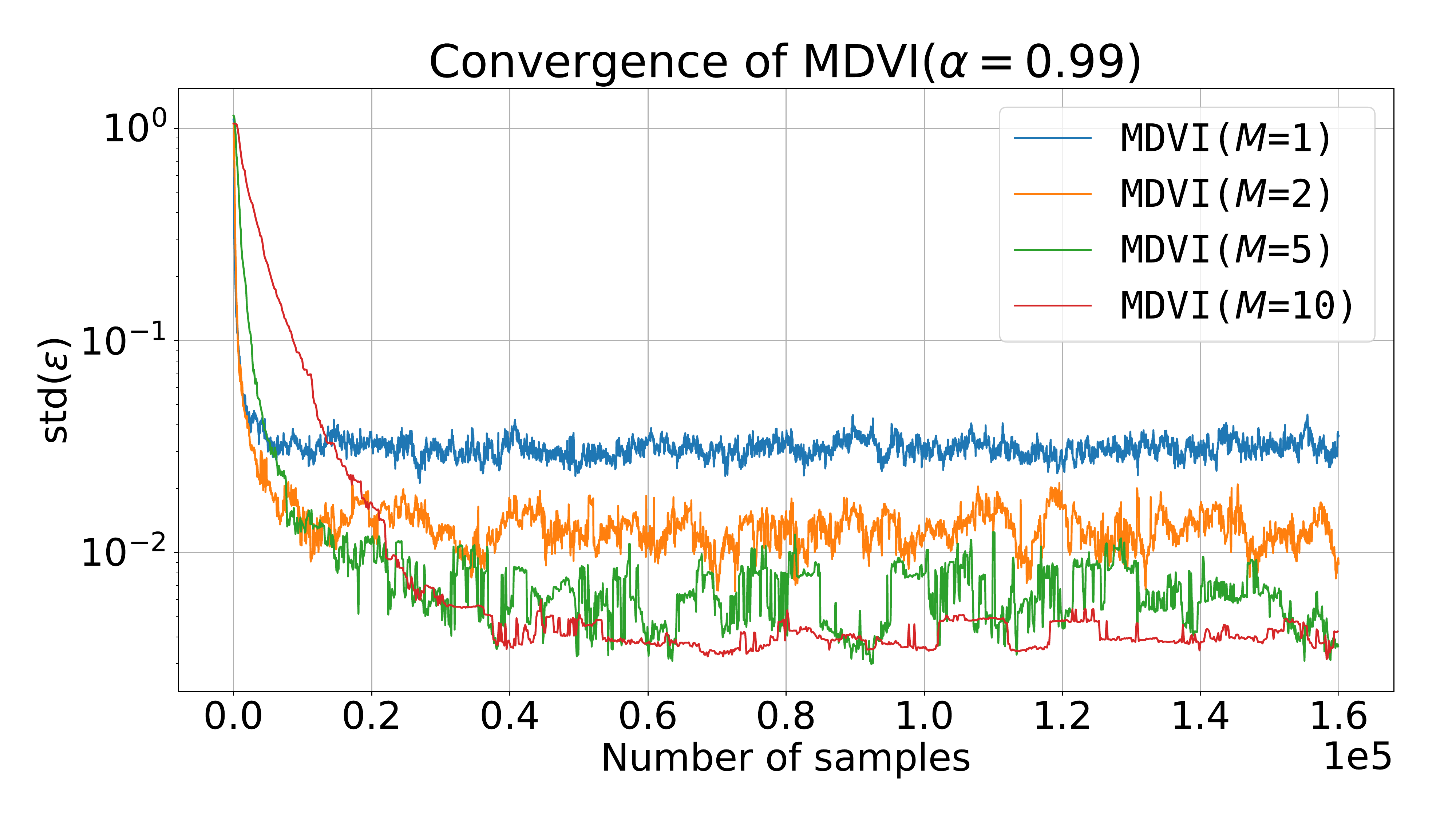}
    \caption{Error of the policy computed by \mdvi in function of the number of samples used, for different values of $M$. \textbf{Left:} mean, \textbf{Right:} standard deviation; estimated over $1000$ MDPs. For this value of $\gamma=0.9$, choosing $\alpha=0.99$ matches the condition $\alpha = 1 - (1-\gamma)^2$.}
    \label{fig:convergence_m}
\end{figure}